\documentclass[11pt]{article}  % Document class with 11pt font
\usepackage{authblk}

\usepackage[letterpaper, margin=1in]{geometry}  % 1-inch margins
\usepackage{setspace}  % For line spacing

\usepackage{graphicx} % Required for inserting images
\usepackage{amsmath, amssymb, stmaryrd, setspace, nicefrac, multicol, amsthm, natbib, geometry,comment, enumerate}
\usepackage{xcolor}
\usepackage{tikz}
\usetikzlibrary{arrows}
\usepackage{tcolorbox}
\usepackage{todonotes}
\usepackage{cleveref}
\usepackage{titlesec}
\usepackage{float}
\usepackage{pgfplots}
\titleformat{\subsubsection}[runin]{\normalfont\normalsize\bfseries}{}{0pt}{}
\usepackage{amsfonts}
\newtheorem{thm}{Theorem}[section]
\newtheorem{lem}[thm]{Lemma}
\newtheorem{cor}[thm]{Corollary}
\newtheorem{prop}[thm]{Proposition}
\newtheorem{fct}[thm]{Fact}
\newtheorem*{claim*}{Claim}

\theoremstyle{definition}

\newtheorem{dfn}[thm]{Definition}

\newtheorem{rem}[thm]{Remark}

\newcommand*{\N}{\ensuremath{\mathbb{N}}}

\newcommand*{\Hyp}{\ensuremath{\mathcal{H}}}
\newcommand{\Ld}[1]{\ensuremath{\textsc{L}\left( #1 \right)}}
\newcommand{\Iso}[1]{\ensuremath{\mathbf{Iso}\left( #1 \right)}}
\newcommand{\VC}[1]{\ensuremath{\textsc{VC}\left( #1 \right)}}
\newcommand{\Fiso}[2]{\mathbf{Iso}_{#1} \left( #2 \right)}

\title{On statistical learning of graphs}

\author{
  Vittorio Cipriani\\\vspace{-0.3cm}
  Vienna University of Technology, Wiedner Hauptstraße 8–10/104, 1040, Vienna, Austria\\
  \texttt{vittorio.cipriani17@gmail.com}
  \and
  Valentino Delle Rose\\ \vspace{-0.3cm}
  ANVUR (Italian National Agency for the Evaluation of the University and Research Systems), Via Ippolito Nievo 35, 00153, Roma, Italy\\
  \texttt{valentino.dellerose@anvur.it}
  \and
  Luca San Mauro\\\vspace{-0.3cm}
  University of Bari (Department of Humanistic Research and Innovation), Piazza Umberto I, 1, 70121, Bari, Italy\\
  \texttt{luca.sanmauro@gmail.com}
  \and
  Giovanni Soldà\\\vspace{-0.3cm}
Ghent University (Deparment of Mathematics: Analysis, Logic and Discrete Mathematics), Krijgslaan 281, 8 9000, Ghent, Belgium\\
  \texttt{giovanni.a.solda@gmail.com}
  }

\date{}

\begin{document}

\maketitle

\begin{abstract}
We study \emph{PAC} and \emph{online learnability} of hypothesis classes formed by copies of a countably infinite graph $G$, where each copy is induced by permuting $G$’s vertices. This corresponds to learning a graph’s labeling, knowing its structure and label set. We consider classes where permutations move only finitely many vertices. Our main result shows that PAC learnability of all such finite-support copies implies online learnability of the full isomorphism type of $G$, and is equivalent to the condition of \emph{automorphic triviality}. We also characterize graphs where copies induced by swapping two vertices are not learnable, using a relaxation of the extension property of the infinite random graph. Finally, we show that, for all $G$ and $k>2$, learnability for $k$-vertex permutations is equivalent to that for $2$-vertex permutations, yielding a four-class partition of infinite graphs, whose complexity we also determine using tools coming from both descriptive set theory and computability theory.
\end{abstract}

\tableofcontents

\section{Introduction}
Graphs are pervasive in computer science and mathematics, as they naturally represent data carrying some structural content. In this paper, we borrow concepts and tools from statistical learning theory to address the problem of \emph{learning countably infinite graphs}. In a nutshell, we will formulate the problem of learning a given graph $G$ by analyzing hypothesis classes consisting of copies of $G$ that are produced by simply permuting the \lq \lq labels\rq\rq\ of its vertices. More formally, following an established convention in computable structure theory (see, e.g.~\cite{Montalban}), we assume that the vertex set of $G$ consists of the set $\N$ of natural numbers and let $\Iso{G}$ denote the family of all \emph{presentations} (or, simply \emph{copies}) of $G$, namely of those isomorphic copies of $G$ which are induced by some permutation of $\N$. We then examine the learnability of natural classes $\Hyp \subseteq \Iso{G}$, either in the PAC or in the online fashion.

To the best of our knowledge, statistical learning theory has primarily concentrated on functions, leaving the study of more complex structures like graphs essentially untouched (a recent notable exception is \cite{Coregliano2024HigharityPL}, which develops a theory of high-arity PAC learning going way beyond the traditional setting). However, as we shall argue in the following paragraphs, looking into the learnability of infinite graphs is a valuable goal.
Specifically, we will see how themes from both classical statistical learning and computable structure theory  naturally emerge in our proposed framework.

\subsubsection*{Related work and motivations}
A central problem in statistical learning theory is that of determining when a given hypothesis class can be learned by an algorithm that has access only to a finite amount of information about an unknown target function within that class. To formalize this problem, several frameworks have been developed, with \emph{PAC learning} and \emph{online learning} (see \Cref{sec:stat-learning}) among the most widely used. A fundamental insight across these frameworks is that a learnable hypothesis class must exhibit some \lq\lq common feature\rq\rq\  among its hypotheses. A major theoretical question, then, is to identify which features ensure learning and which, on the contrary, hinder it. For example, it is well-known that the class of all axis-aligned rectangles in $\mathbb{R}^2$ can be learned in the PAC model, while the class of all convex polygons cannot (see, e.g.~\cite{Kearns-Vazirani}). Our work continues along this line of inquiry, by focusing on hypothesis classes whose  common feature is, roughly speaking, that they encode the very same structure in different ways (meaning that they are presentations of the same graph).

Studying structures up to a specified notion of similarity is also a common thread in \emph{computable structure theory}, the area of mathematical logic that studies the algorithmic complexity of mathematical structures (see \cite{Montalban} for a comprehensive overview of this subject). One of the most popular measures of complexity of a given structure $\mathcal{S}$ is its \emph{degree spectrum}, denoted $\mathbf{DgSp}(\mathcal{S})$, which is defined as the collection of Turing degrees of all its isomorphic copies---in fact, given that computable structure theory often assumes that structures have domain $\mathbb{N}$, one may also regard the degree spectrum of $\mathcal{S}$ as the collection of Turing degrees of elements from $\Iso{S}$. The framework we propose here provides an alternative approach to measuring the complexity of the copies of a graph, employing tools from statistical learning theory.

% want a measure ofcomplexity that is independent of the particular !-presentation

%The complexity of a structure is often studied up to \emph{isomorphism} or, as in \cite{MR2664124}, up to \emph{automorphism}. As the authors motivate, automorphisms allow the study of the symmetries of a structure and, to measure how complicated these automorphisms are, they introduce the notion of \emph{automorphism spectrum} of a structure, namely the set of Turing degrees of all of its automorphic copies (except the trivial copy given by the identity). The framework we propose in this work can be viewed as an alternative method for measuring the complexity of the automorphic copies of a graph using tools from statistical learning theory.

Our approach is also related to the emerging field of \emph{online structure theory} (see \cite{foundations1,foundations2}). While traditional computable structure theory operates in an offline setting, where a Turing machine has unrestricted access to a structure, online structure theory deals  with algorithmic
situations in which the input arrives one bit at a time but decision has to be
made instantly (a similar constraint applies to online learning). In \cite{foundations1,foundations2} the authors provide compelling motivations as to why  online structures are best understood as those structures with domain $\mathbb{N}$ and having 
operations and relations which are primitive recursive.

Finally, we examine the complexity of determining whether certain classes of graphs are learnable within a given paradigm, aiming to quantify the difficulty of the learning process. To this end, we draw on tools from \emph{descriptive set theory}—the field studying the complexity and structure of definable sets in topological spaces. In particular, we use \emph{Wadge reducibility}, which provides a framework for comparing the relative complexity of subsets of such spaces. The complexity of learnable hypothesis classes has been studied in the literature, though only within the realm of computable hypothesis classes. To the best of our knowledge, this is the first work to analyze the complexity of learnable classes from a topological perspective, without restricting to computability. That said, we will also examine the computable case in due course. Prior studies have focused on settings such as inductive inference (see \cite{beros, bberos}) and (variants of) PAC learning (see \cite{decidingVC, calvert, Sterkenburg2022-STEOCO-5, contfeatures}); to our knowledge, this is the first to consider the complexity of online learnable hypothesis classes as well. %Also we mention that  This investigation aligns with prior work in the literature. In  \cite{decidingVC} the author studied, in the general setting of function learning, the complexity of PAC learnable hypothesis classes, while in \cite{calvert} the study concerned specific classes given by computable enumerations of computable $\Pi_1^0$ classes. Related issues have also been explored in the context of 

% \lq\lq simple\rq\rq\ structures in the online context correspond to primitive recursive one with domain the natural numbers. 

% a scenario where a learner must determine whether two vertices are connected by an edge, given only a finite portion of the graph. 

% while computable structure theory deals with an \emph{offline} situation, where the Turing machine may access the desired structure with no resource bounds, in our framework, a learner must say wether two vertices are connected by an edge with the knowledge of just a finite portion of the graph. This scenario is the one studied by online structure theory, which focuses on structures that are, intuitively speaking, computable \emph{without delay}. In \cite{foundations1,foundations2} the authors provide several motivations on why \lq\lq simple\rq\rq\ structures in the online context correspond to primitive recursive one with domain the natural numbers. 

\subsubsection*{Our results} The results presented in this work fully determine the landscape of online and PAC learning of families of presentations of countable graphs induced by permutations of finite support, as depicted in \Cref{fig:landscape} below. In particular, we denote by $\Fiso{k}{G}$ the restriction of $\Iso{G}$ to the presentations of $G$ induced by those permutations which move at most $k$ many vertices. We prove the following:
\begin{itemize}
    \item \Cref{thm:at-implies-online} and \Cref{thm:finite-pac-implies-at} establish that, for all graphs $G$, $\bigcup_k \Fiso{k}{G}$ is PAC learnable if and only if $\Iso{G}$ is online learnable, with both these conditions being equivalent to the fact that $G$ is \emph{automorphically trivial} (see \Cref{definition:automorphicallytrivial} below). 
    \item \Cref{thm:ar-equals-anl} characterizes \emph{absolutely non-learnable graphs}, namely those graphs in which not even $\Fiso{2}{G}$ is PAC learnable, by means of the notion of \emph{almost randomness} (introduced in \Cref{def:almost-random}). This allows us to conclude that the random graph is absolutely non-learnable. 
    \item We show that, for any graph $G$, PAC learnability of $\Fiso{2}{G}$ implies that $\Fiso{k}{G}$  is PAC learnable, for every $k \in \N$ (\Cref{thm:2kPAC}). Also, an analogous result holds for the online learning case (\Cref{thm:2konline}). As a consequence, we can divide all graphs into the following four classes:
    \begin{enumerate}
        \item \emph{online learnable} graphs, i.e.~those graphs $G$ with $\Iso{G}$ online learnable;
        \item \emph{weakly online learnable} graphs, i.e.~those graphs $G$ for which $\Fiso{k}{G}$ is online learnable, for every $k \in \N$;
        \item \emph{weakly PAC learnable} graphs, i.e.~those graphs $G$ such that, for every $k \in \N$, $\Fiso{k}{G}$ is PAC learnable;
        \item \emph{absolutely non-learnable} graphs, i.e.~those graphs $G$ with $\Fiso{2}{G}$ not PAC learnable.
    \end{enumerate} 
    \item Clearly, it holds that: $\text{online learnable} \Rightarrow \text{weakly online learnable} \Rightarrow \text{weakly PAC learnable}.$ We prove  that, in general, none of the converse implications hold (\Cref{prop:kequiv} and \Cref{prop:onlinepacdifferent}).
\item  \Cref{thm:firstcomplexity} establishes that (weakly) online learnable and weakly PAC learnable graphs are $\boldsymbol{\Sigma}_2^0$-complete while absolutely non-learnable graphs are $\boldsymbol{\Pi}_2^0$-complete. 
\Cref{thm:firstcomplexityD2} and \Cref{thm:secondcomplexityD2} prove that the class of graphs that are weakly online but not online learnable and the class of graphs that are weakly PAC learnable but not weakly online learnable are $D_2(\boldsymbol{\Sigma}_2^0$)-complete: the definitions of such complexity classes will be given in \Cref{sec:complexityprel}. The effective versions of these theorems will be proved in \Cref{theorem:firsteffectivitazion}
%, which encompasses the effective analogs of the theorems just discussed, strengthens \cite[Theorem 4.1]{decidingVC} by identifying a proper subclass of computable hypothesis classes that are PAC learnable (defined in \Cref{sec:complexitycomputable}), yet whose complexity matches that of the general case.
\end{itemize}

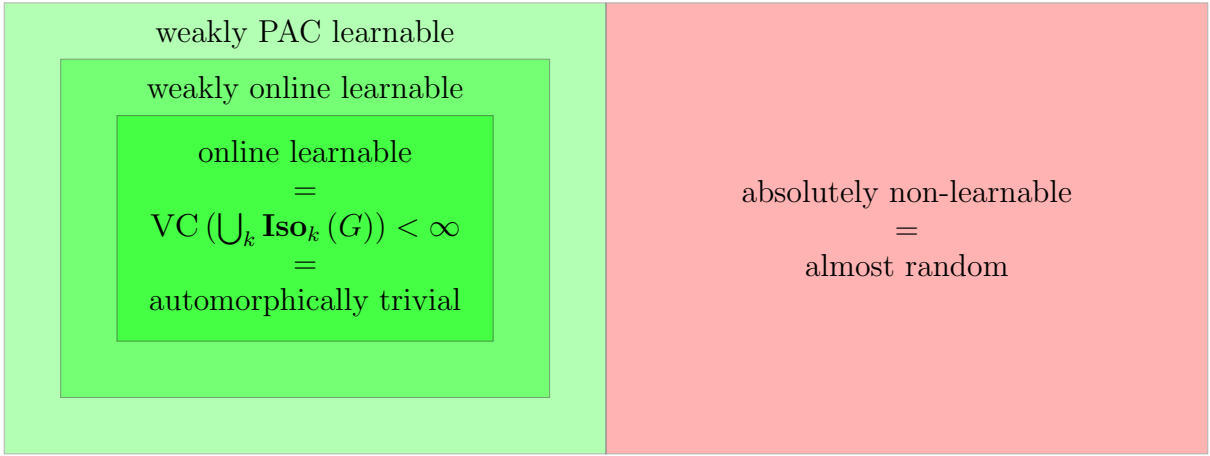
\begin{figure}[h]
    \centering
    \begin{tikzpicture}
% Outer Rectangle - "Weakly PAC learnable"
\node[draw, rectangle, minimum width=8cm, minimum height=6cm, fill=green, opacity=0.3] (outer) at (0,0) {};
\node[below=0.1] at (outer.north) {\large weakly PAC learnable};

% Middle Rectangle - "Weakly online learnable"
\node[draw, rectangle, minimum width=6.5cm, minimum height=4.5cm, fill=green, opacity=0.35] (middle) at (0,0) {};
\node[below=0.1] at (middle.north) {\large weakly online learnable};

% Inner Rectangle - "online learnable = PAC learnable = automorphically trivial"
\node[draw, rectangle, minimum width=5cm, minimum height=3cm, fill=green, opacity=0.4] (inner) at (0,0) {};
\node at (inner) {\parbox{4.5cm}{\centering \large online learnable \\ $=$ \\ $\VC{\bigcup_k \Fiso{k}{G}} < \infty$ \\ $=$ \\ automorphically trivial}};

% Right Section - "absolutely non-learnable = almost random"
\node[draw, rectangle, minimum width=8cm, minimum height=6cm, fill=red, opacity=0.3] (right) at (8,0) {};
\node at (right) {\parbox{6cm}{\centering \large absolutely non-learnable \\ $=$ \\  almost random}};

\end{tikzpicture}
 
    \caption{The landscape of statistical learning of graphs.}
    \label{fig:landscape}
\end{figure}

Here are some further comments on our results. By \cite[Theorem 3.1]{notuniversalgraphs} automorphically trivial graphs are precisely those that are \emph{punctually categorical} graphs---that is, those having a  unique primitive recursive copy with domain $\mathbb{N}$, up to primitive recursive isomorphism. The coincidence of these seemingly unrelated notions---punctual categoricity, which is an effective notion, and online/PAC learnability, in which no effective requirement is imposed---is unexpected. Nonetheless, our findings may serve as further evidence that primitive recursiveness is an apt model for capturing online scenarios. We highlight that the equivalence between online (PAC) learnable graphs and automorphically trivial graphs may also lead to the wrong conclusion that if a graph $G$ is sufficiently \lq\lq symmetric\rq\rq\, then $G$ (or rather, a suitable collection of copies of $G$) should be learnable.
However, this intuition fails: \Cref{sec:abs-non-learnable} implies that \lq\lq highly symmetric\rq\rq\ graphs, such as the \emph{countably infinite random graph}\footnote{Here, we refer to the well-known fact that the random graph $\mathcal{R}$ is \emph{homogenous}, i.e., every isomorphism between two induced subgraphs of $\mathcal{R}$ can be extended to an automorphism \cite{Erdos}.}, are absolutely non-learnable.

Finally, it is worth noticing that the study of what we called \lq\lq weak\rq\rq\ notions of learnability is readily motivated by the equivalence between learnability of the whole isomorphism type and automorphic triviality. Indeed, in view of this correspondence, the learning process seems to be facilitated by more local properties than isomorphism, which, by contrast, is \lq\lq global\rq\rq\ in nature.

%Finally, it is worth noticing that the study of weak online (PAC) learnability is naturally motivated by the equivalence between (online) PAC learnable graphs and automorphically trivial graphs. Indeed, this correspondence suggests that to have a deeper understanding of the complexity of statistical learning of graphs one should also consider notions that are more  \lq\lq sensitive\rq\rq\ to \lq\lq local\rq\rq\ phenomena than isomorphism. In fact, isomorphism captures \emph{global} properties of the graph while the learning process may be greatly affected by \emph{local} properties, as at each step, the learner has access only to a finite part of the graph. Hence, it is natural to consider what we call the \emph{weak} variant of online (PAC) learning, where the hypothesis class is restricted to \lq\lq simpler\rq\rq\ copies, i.e., those generated from permutation with finite support.

\subsubsection*{Structure of the paper} In \Cref{sec:pre} we give the graph theoretical notations used in the paper, we introduce our framework for online and PAC learning of graphs. In \Cref{sec:online-learnable} we prove the equivalence between online learnable, PAC learnable and automorphically trivial graphs and in \Cref{sec:abs-non-learnable} we characterize the absolutely non-learnable graphs in terms of almost-random graphs. \Cref{sec:finitesupport} introduces the notions of weak online and weak PAC learnability and studies their properties. Finally, in \Cref{sec:complexity},  after giving all the necessary preliminaries about descriptive set theory and computability theory, we give a complete picture of the Wadge complexity of classes of graphs learnable in any of the learning paradigm mentioned above discussing also the effective analog.

\section{Preliminaries} \label{sec:pre}
\subsection{Graph theoretical notation}
Unless explicitly specified, we consider only graphs which are countable, undirected and simple: in particular, for any graph $G$, we always assume, without loss of generality, that its set of vertices $V(G)$ coincides with the set $\N$ of natural numbers. For a vertex $v \in V(G)$ we let $N(v):=\{w : (v,w) \in E(G)\}$ be its set of \emph{neighbors}. Two examples of countably infinite graphs that will play a role in the next pages are the \emph{infinite clique} $K_\N$ (namely the graph in which any pair of vertices is an edge) and its complement, the \emph{infinite anticlique} $\overline{K_\N}$. Given two graphs $G$ and $H$ we denote by $G \oplus H$ the \emph{disconnected union} of $G$ and $H$, i.e., the graph consisting of a copy of $G$ and $H$ with no edge connecting them.

 We denote by $S(\N)$ the \emph{symmetric group} of $\N$, consisting of all permutations $\N \to \N$. 
 
 \begin{dfn}
     A graph $H$ is a \emph{presentation} (or, simply, a \emph{copy}) of a graph $G$ if there is a permutation $f \in S(\N)$ such that for every $u,v \in \N$, $(f(u),f(v)) \in E(H)$ if and only if $(u,v) \in E(G)$. Moreover, for a graph $G$, we denote by $\Iso{G}$ the family of its presentations, to which, with a slight abuse of terminology, we refer sometimes as the \emph{isomorphism type} of $G$.
 \end{dfn}
    It is understood that every element of $\Iso{G}$ is induced by some $h \in S(\N)$: when this is the case, we sometimes denote that copy as $h(G)$. For $h \in S(\N)$, the \emph{support} of $h$ is the set $\{x \colon h(x) \ne x \}$: for $k \in \mathbb{N}$, let us denote by $\Fiso{k}{G}$ the restriction of $\Iso{G}$ to the class of copies of $G$ induced by permutations with support of size at most $k$. The rest of the notions and notation we use are quite standard and follows mostly the textbook \cite{Diestel-book}. 

\subsection{Statistical learning of graph} \label{sec:stat-learning}
Our focus is the study of statistical learnability on countably infinite graphs:  given such a graph $G$, we aim at learning families of its copies $\Hyp \subseteq \Iso{G}$, either in the \emph{PAC} or \emph{online} fashion. The reader is referred to \cite{shalev-shwartz-ben-david} for the basic notions and techniques in both these frameworks. Here, we will briefly rephrase both learning frameworks, when adapted to our case. First notice that, given a graph $G$ and a family $\Hyp \subseteq \Iso{G}$ (which we will call the \emph{hypothesis class}), we will regard any \emph{hypothesis} $H \in \Hyp$ as a function $\chi_H: \N \times \N \to \{0,1\}$ such that $\chi_H(u,v) = 1$ if and only if $(u,v) \in E(H)$. 
\subsubsection*{Online learning}  The framework for online learning can be described as a game between a \emph{Learner} ($\mathbf{L}$) and an \emph{Opponent} ($\mathbf{O}$) on a fixed class $\Hyp \subseteq \Iso{G}$, which is known to both players. Before the game starts, $\mathbf{O}$ secretly chooses a hypothesis $H \in \Hyp$. Then, at each round:
\begin{enumerate}
    \item $\mathbf{O}$ picks some $(u,v) \in \N \times \N$.  
   \item $\mathbf{L}$ predicts whether $(u,v)$ is in $ E(H)$. We denote this with $\mathbf{L}((u,v))=i$ where $i=0$ if  $\mathbf{L}$ predicts $(u,v) \notin E(H)$ and $k=1$ if  $\mathbf{L}$ predicts $(u,v) \in E(H)$.
   \item $\mathbf{O}$ reveals wether $(u,v)$ is in $E(H)$.
\end{enumerate}
The game goes on for infinitely many rounds. We say that $\Hyp$ is \emph{online learnable} if there is a $d \in \N$ and a strategy for  $\mathbf{L}$ such that, no matter which $H \in \Hyp$ has been chosen by $\mathbf{O}$ nor the order in which $\mathbf{O}$ picks the elements in $\N \times \N$, $\mathbf{L}$ makes at most $d$ mistakes. In the seminal paper \cite{Littlestone}, Littlestone gave a combinatorial characterization of the optimal mistake bound achievable in the above game on a given class $\Hyp$, which is known as \emph{Littlestone dimension} and will be denoted by $\Ld{\Hyp}$: in other words, $\Hyp$ is online learnable with at most $d$ mistakes if and only if $\Ld{\Hyp} \le d$. 
\begin{rem} \label{rem:ld-subclasses}It is clear that a strategy for $\mathbf{L}$ which makes at most $d$ mistake on a class $\Hyp$ can be used to online learn any subclass $\Hyp' \subseteq \Hyp$ with the same bound on the number of mistakes: in other words, if $\Hyp' \subseteq \Hyp$, then $\Ld{\Hyp'} \le \Ld{\Hyp}$. \end{rem}
We will extensively use the following alternative characterization of the Littlestone dimension due to Shelah \cite{Shelah}, as rephrased in \cite{Alon-et-al}. 
\begin{dfn}
    Let $k \in \N$. We say that $\Hyp$ \emph{contains} $k$ \emph{thresholds} it there are $(u_1,v_1), \dots, (u_{k-1},v_{k-1}) \in \N \times \N$ and $H_1, \dots, H_k \in \Hyp$ such that for all
    $i, j \le k$, it holds that $(u_j,v_j) \in E(H_i) \iff i \le j.$
\end{dfn}
%\commgio{}{The definition of threshold was wrong. Should we say this to the referee?}

\begin{fct}[\cite{Shelah}, \cite{Alon-et-al}] \label{fct:thresholds}
    For any class $\Hyp$,  $\Hyp$ is online learnable if and only if $\Hyp$ does not contain infinitely many thresholds. 
  %  \begin{enumerate}
   %     \item if $\Ld{\Hyp} \ge d$, then $\Hyp$ contains $\lfloor \log d \rfloor$ thresholds; 
  %      \item if $\Hyp$ contains $d$ thresholds, then $\Ld{\Hyp} \ge \lfloor \log d \rfloor$.
 %   \end{enumerate}
   % Thus,
\end{fct}
\subsubsection*{PAC learning}
Another widely studied notion of learnability is the one given within the framework of \emph{PAC learning}, introduced by Valiant \cite{Valiant}. We will directly use as definition of PAC learnability a fundamental characterization proved in \cite{Blumer-et-al}, which relates Valiant's framework with a combinatorial measure defined by Vapnik and Chervonenkis \cite{VC}.
We say that $\Hyp \subseteq \Iso{G}$ \emph{shatters} a $d$-tuple $(u_1,v_1), \dots, (u_d,v_d) \in \N \times \N$ if and only for every $\tau \in \{0,1\}^d$, there is $H \in \Hyp$ such that for every $1 \le i \le d$, 
$(u_i,v_i) \in E(H) \iff  \tau(i)=1$.

The \emph{Vapnik-Chervonenkis dimension} (\emph{VC-dimension} for short) of $\Hyp$, denoted by $\VC{\Hyp}$ is the largest $d$ such that there exists $d$ distinct points shattered by $\Hyp$. If $\Hyp$ shatters $d$ points for any $d$, then we say that $\VC{\Hyp}= \infty$. Finally, we say that $\Hyp$ is \emph{PAC learnable} if $\VC{\Hyp}$ is finite.
\begin{rem} \label{rem:vc-subclasses} It is clear from the definition that, if $\Hyp' \subseteq \Hyp$, then $\VC{\Hyp'} \le \VC{\Hyp}$.\end{rem}

We conclude by noticing that PAC learnability is a strictly less demanding property than online learnability.

\begin{fct}[\cite{Littlestone}]
\label{fact:onlineimpliespac}
    For every hypothesis class $\Hyp$, $\VC{\Hyp} \le \Ld{\Hyp}$: hence, every online learnable class is PAC learnable. On the other hand, in general, there are hypothesis classes which are PAC learnable but not online learnable.
\end{fct}

\section{A characterization of online learnable graphs} \label{sec:online-learnable}
The aim of this section is to characterize those graphs whose entire isomorphism type is online learnable (and, hence, PAC learnable by \Cref{fact:onlineimpliespac}).
\begin{dfn}
    A graph $G$ is \emph{online learnable} if $\Iso{G}$ is online learnable.
\end{dfn}
On the one hand, we prove that online and PAC learnability are the same when learning the whole isomorphism type of a graph $G$: in fact, online learnability of a graph $G$ actually boils down to PAC learnability of all copies of $G$ induced by permutations of its vertices with finite support, i.e.~of $\bigcup_{k \in \N} \Fiso{k}{G}$. On the other hand, we also show that we can learn the whole isomorphism type only of graphs having, roughly speaking, an \emph{extremely simple} structure: indeed, we show that online learnable graphs coincide with \emph{automorphically trivial} ones.

Let us begin by reviewing how, in general, automorphically trivial structures are defined.
\begin{dfn}[\cite{knight_autotrivial}]
\label{definition:automorphicallytrivial}
A structure $\mathcal{S}$ is \emph{automorphically trivial} if there is a finite subset $S_0$ of its domain $S$, such that every permutation of $S$ fixing $S_0$ pointwise is an automorphism.
\end{dfn}
When focusing on the case of graphs, we have the following useful characterization.
\begin{fct}[{\cite[Example 1.5]{miller}}]
\label{fact:automorphicallytrivialgraphs}
 The automorphically trivial graphs $G$ are exactly those obtained as follows. We choose a partition $V(G) = S_0 \sqcup S_1$ with $S_0$ being finite and some $S_0' \subseteq S_0$. Then we let \[E(G)=E(G[S_0]) \cup \{(v,w): v \in S_0' \text{ and } w \in S_1\} \cup E(G[S_1])\] where either $G[S_1]$ is isomorphic to $ K_\mathbb{N}$ or $G[S_1]$ is isomorphic to $ \overline{K_\mathbb{N}}$. Notice that automorphically trivial graphs are clearly closed under complement.
\end{fct}

The announced characterization arises from the combination of several results.  Our strategy is summarized in the following diagram.
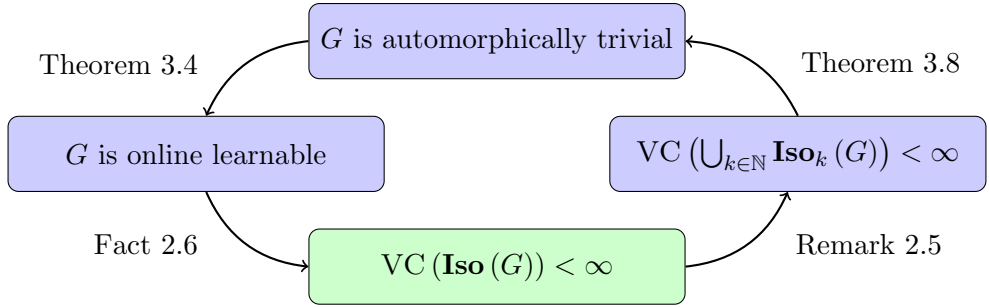
\begin{center}
    \begin{tikzpicture}
      \node[rectangle, rounded corners, draw=black, minimum height=1cm, minimum width=5cm, fill=blue!20, text centered] (at) at (0,0) {$G$ is automorphically trivial};
      
      \node[rectangle, rounded corners, draw=black, minimum height=1cm, minimum width=5cm, fill=green!20, text centered] (pac) at (0,-3) {$\VC{\Iso{G}} < \infty$};
      
      \node[rectangle, rounded corners, draw=black, minimum height=1cm, minimum width=5cm, fill=blue!20, text centered] (online) at (-4,-1.5) {$G$ is online learnable};
      
      \node[rectangle, rounded corners, draw=black, minimum height=1cm, minimum width=5cm, fill=blue!20, text centered] (finite-pac) at (4,-1.5) {$\VC{\bigcup_{k \in \N} \Fiso{k}{G}} < \infty$};

      \draw[thick, ->] (at.west) to[bend right] node[midway, left=0.5cm] {\Cref{thm:at-implies-online}} (online);
      
      \draw[thick, ->] (online) to[bend right]  node[midway, left=0.5cm]{\Cref{fact:onlineimpliespac}} (pac.west);
      
      \draw[thick, ->] (pac.east) to[bend right] node[midway, right=0.5cm] {\Cref{rem:vc-subclasses}} (finite-pac);
      
      \draw[thick, ->] (finite-pac.north) to[bend right] node[midway, right=0.5cm] {\Cref{thm:finite-pac-implies-at}} (at.east);
    \end{tikzpicture}
\end{center}
We begin by proving that every automorphically trivial graph is online learnable.
\begin{thm}
\label{thm:at-implies-online}
If $G$ is an automorphically trivial graph, then $G$ is online learnable.
	\end{thm}
\begin{proof}
	Assume that $G$ is automorphically trivial. Therefore, we can refer to a partition of $\N$ into two sets $S_0 \sqcup S_1$ and a subset $S_0'\subseteq S_0$, where these sets comes from the description of $V(G)$ given by \Cref{fact:automorphicallytrivialgraphs}. 
 The opponent $\mathbf{O}$ chooses some hypothesis in $\Iso{G}$ and we denote by $v_s,w_s$ the vertices presented by $\mathbf{O}$ to the learner $\mathbf{L}$ at round $s$ of the learning game. We define $\mathbf{L}$ by letting  $\mathbf{L}((v_{s},w_{s}))=1$ if $S_1$ is isomorphic to $K_\mathbb{N}$ and $\mathbf{L}((v_{s},w_{s}))=0$ if $S_1$ is isomorphic to $\overline{K_\mathbb{N}}$ unless 
 \[|\{(v_s,w_t) : t \leq s \text{ and } \mathbf{O}((v_s,w_t))\neq \mathbf{L}((v_s,w_t))\}| > |S_0|,\]
 in which case, for any $t \geq s$,  if $S_1$ is isomorphic to $K_\mathbb{N}$ we let  $\mathbf{L}((v_{s},w_{t}))=0$, while if $S_1$ is isomorphic to $\overline{K_\mathbb{N}}$  we let  $\mathbf{L}((v_{s},w_{t}))=1$. 
 
We claim that $|E(G[S_0])| +  |S_0|^2$ is an upper bound to the number of mistakes made by $\mathbf{L}$.
Notice that:
\begin{itemize}
    \item[(i)] if $S_1$ is isomorphic to ${K_\mathbb{N}}$ then for any $v \in S_0\smallsetminus S_0'$, $|N(v)|<|S_0|$ while for any $w \in \N\smallsetminus (S_0\smallsetminus S_0')$ we have that $|N(w)|=\infty$ and $|\N\smallsetminus N(w)| \leq |S_0|$;
       \item[(ii)] if $S_1$ is isomorphic to $\overline{K_\mathbb{N}}$ then for any $v \in S_0'$ we have that $|N(v)|=\infty$ and $|\N\smallsetminus N(v)| < |S_0|$ while for any $w \notin S_0'$ we have that $|N(v)| \leq |S_0|$.
\end{itemize}
The only possible mistakes made by $\mathbf{L}$ are on a pair $(v,w)$ where either $v,w \in S_0$ or  $v \in S_0\smallsetminus S_0'$ and $w \in S_1$ (if (i) holds) or  $v \in  S_0'$ and $w \in S_1$ (if (ii) holds).
Combining (i) and (ii) one obtains $\mathbf{L}$ makes at most $|E(G[S_0])|$ mistakes due to pairs $(v,w) \in S_0 \times S_0$. In the other cases, if (i) holds, for every $v \in S_0\smallsetminus S_0'$, $\mathbf{L}$ can make $|S_0|$ mistakes, while if (ii) holds, for every $v \in S_0'$, $\mathbf{L}$ can make at most $|S_0|$ mistakes. Summing up, it is straightforward to check that $|E(G[S_0])| +  |S_0|^2$ is an upper bound to the number of mistakes made by $\mathbf{L}$.
\end{proof}
It remains to prove that if $\VC{\bigcup_{k \in \N} \Fiso{k}{G}}$ is finite, then $G$ must be automorphically trivial. To this end, we need some preliminary lemmas. First, we observe that the finiteness VC dimension of $\Iso{G}$ is closed under taking induced subgraphs.  
\begin{lem}\label{lem:VC-induced-subgraphs}
    Let $G$ be a graph with $\VC{\Iso{G}} \le d$. Then, for every induced subgraph $H$ of $G$, $\VC{\Iso{H}} \le d$.
\end{lem}
\begin{proof}
    Assume that there is an induced subgraph $H$ of $G$ such that $\VC{\Iso{H}} > d$: hence, there are $(x_0, y_0) \ne \dots \ne (x_d, y_d) \in \N \times \N$ such that 
    $$\forall \tau \in \{0,1\}^{d+1} \exists h \in S(\N) \ (h(x_i),h(y_i)) \in E(h(H)) \iff \tau(i)=1.$$ Since $H$ is an induced subgraph of $G$, $(h(x_i),h(y_i)) \in E(h(H))$ if and only if  $(h(x_i),h(y_i)) \in E(h(G))$, meaning that $\VC{\Iso{G}} > d$, too.
\end{proof}

Next, we need to introduce some notation. Recall that given a graph $G$, a \emph{matching} $M$ in $G$ is a set of edges such that no two edges share common vertices. We call $M_d$ the graph formed by a matching with $2d$ vertices and $2d$ isolated vertices and $N_d$ the set of graphs with $2d+1$ vertices such that one vertex is connected to $d$ vertices and disconnected from $d$ vertices. Examples for $d=2$ are depicted in \Cref{fig:m2-n2} below.
\begin{figure}[h!]
    \centering
    \begin{tikzpicture}
        %\node at (-1.5, 1) {$M_2$};
        \fill (-3,0) circle (2pt); 
        \fill (-2,0) circle (2pt); 
        \draw[thick] (-2,0) -- (-3,0);
        \fill (0,0) circle (2pt);
        \fill (-1,0) circle (2pt);
        \draw[thick] (0,0) -- (-1,0);
        \fill (1,0) circle (2pt);
        \fill (2,0) circle (2pt);
        \fill (3,0) circle (2pt);
        \fill(4,0) circle (2pt);

        %\node at (3, 1) {$N_2$};
        %\begin{comment}
        \fill (8,0) circle (2pt);
        \fill (7,.5) circle (2pt);
        \fill (7,-.5) circle (2pt);
        \draw[thick] (8,0) -- (7,.5);
        \draw[thick] (8,0) -- (7,-.5);
        \fill (9,.5) circle (2pt);
        \fill (9,-.5) circle (2pt);
        %\end{comment}
        
    \end{tikzpicture}
    \caption{The graphs $M_2$ (to the left) and $N_2$ (to the right).}
    \label{fig:m2-n2}
\end{figure}
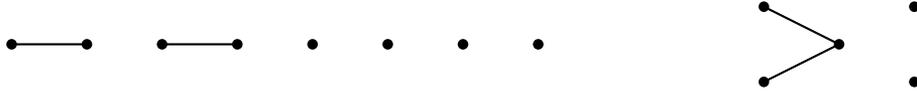

We observe that the family of copies of graphs of the form $M_d$ and $N_d$ with bounded support have large VC dimension. \begin{lem}\label{lem:forbidden-subgraphs} For every $d$, it holds that:
\begin{itemize}
    \item $\VC{ \Fiso{4d}{M_d}}\ge d$,
    \item $\VC{ \Fiso{4d}{\overline{M_d}}}\ge d$,
    \item$\VC{ \Fiso{2d}{N_d}} \ge d$.
\end{itemize}     
%\end{enumerate}

\end{lem}
\begin{proof}
    We begin by proving the statement for $M_d$. Let us fix the following copy of $M_d$ as standard: $V(M_d) = \{ 0, \dots, 4d-1 \}$ and $E(M_d) = \{(2i, 2i+1): i <d \}$.
    We show that the family of copies of $G$ induced by the permutations of support at most $4d$ shatters the pairs of vertices $(0,1), \dots, (2d-2, 2d-1)$. For every $n \leq d$, let $h_n \in S(\N)$ be the permutation swapping vertices $2n$ and $2n+1$ with, respectively, vertices $2(n+d)$ and $2(n+d)+1$ and which acts as the identity on all other vertices (hence the support of $h_n$ is contained in $V(M_d)$).
    Now, if $\tau \in \{0,1\}^d$ is such $\tau(i) = 0 $ if and only if $i \in \{i_1 < \dots < i_k\}$ (with $0 \le i_j < d$), we let $h_{\tau} = h_{i_1} \circ \dots \circ h_{i_k}$: then, for every $0 \le i < d-1$ we get that
    $$(h_{\tau}(2i), h_{\tau}(2i+1)) \in  E(h_{\tau}(M_d)) \iff \tau(i) = 1,$$
    as required. Notice that each $h_n$ has support 4, hence $h_{\tau}$ has support at most $4d$. The case $\overline{M_d}$ is proven similarly.

    Next, we prove the claim for $N_d$. Again, we fix a standard copy of $N_d$ by letting $V(N_d) = \{0,\dots, 2d\}$ and $E(N_d) = \{(0,i): 1 \le i \le d \}$. Our goal is to show that $\Fiso{2d}{N_d}$ shatters the pairs of vertices $(0,1), \dots, (0,d)$. For every $n$, we now let $h_n \in S(\N)$ be the permutation which swaps vertex $n$ with vertex $n+d$ and which acts as the identity on all other vertices (again, the support of $h_n$ is contained in $V(N_d)$).
    As before, to every word $\tau \in \{0,1\}^d$ we associate the permutation $h_{\tau} = h_{i_1} \circ \dots \circ h_{i_k}$, where $0 \le i_1 < \dots < i_k <d$ are all and only positions $i$ such that $\tau(i) = 0$. Hence, for every $0 \le i < d$ we obtain
    $$(h_{\tau}(0), h_{\tau}(i)) \in  E(h_{\tau}(M_d)) \iff \tau(i) = 1,$$
    as we want. Finally, since each $h_n$ has support 2, $h_{\tau}$ has support at most $2d$.
\end{proof}

Finally, we show that every graph which is not automorphically trivial must contain as an induced subgraph infinitely many graphs as those considered above.
\begin{lem}\label{lem:at-vs-forbidden-subgraphs}
    For every infinite graph $G$, exactly one of the following holds.
    \begin{enumerate}
        \item $G$ is automorphically trivial; 
        \item for every $d$, $G$ admits $M_d$ as an induced subgraph or for every $d$, $G$ admits a graph from $N_d$ as induced subgraph or for every $d$, $G$ admits $\overline{M_d}$ as induced subgraph.
    \end{enumerate}
\end{lem}
\begin{proof}
    Notice that if a graph is automorphically trivial, then \Cref{thm:at-implies-online} and \Cref{fact:onlineimpliespac} ensure that $\VC{\Iso{G}}$ is finite: hence, by \Cref{lem:VC-induced-subgraphs} together with \Cref{lem:forbidden-subgraphs}, such graph does not admit infinitely many $M_d$ or $N_d$ as induced subgraphs.

    So we only need to see that if a graph does not satisfy condition $2$ then it is automorphically trivial. In particular, there is $n$ such that $G$ embeds no graph from $N_n$. This implies that every vertex is either connected to at most $n$ vertices or it is disconnected from at most $n$ vertices. In particular, $G$ has cofinitely many vertices of finite degree or just finitely many vertices of finite degree. We deal with the first case, i.e.\ cofinitely many vertices of finite degree. We delete the finitely many vertices of infinite degree: if infinitely many vertices of non-zero degree are left, then $G$ admits an infinite matching as an induced subgraph, contradiction. Hence, the remaining graph contains only finitely many edges and hence it is automorphically trivial.

    The other case is analogous because automorphically trivial graphs are closed under complementation.
\end{proof}
Thus, putting the above lemmas together, we obtain the promised result.
\begin{thm} \label{thm:finite-pac-implies-at} Let $G$ be such that $\bigcup_{k \in \N} \Fiso{k}{G}$ is PAC learnable. Then $G$ is automorphically trivial.
\end{thm}
\begin{proof}
    Assume that $G$ is not automorphically trivial. Then $G$ must satisfy condition 2 of \Cref{lem:at-vs-forbidden-subgraphs}: whatever of the three cases applies, \Cref{lem:VC-induced-subgraphs} and \Cref{lem:forbidden-subgraphs} together imply that  $\bigcup_{k \in \N} \Fiso{k}{G}$ is not PAC learnable.
\end{proof}

%\section{Restricting the class of hypotheses} As the previous section shows, absolute learnability encapsulates different learning criteria. It is natural to ask what happens when we consider online or PAC learnability of all isomorphic copies of $G$ induced by automorphisms of a fixed finite support, i.e., $\Fiso{k}{G}$.

\section{Absolutely non-learnable graphs}\label{sec:abs-non-learnable}
In this section, we analyze the opposite end of the learning spectrum, that is, those graphs in which even the simplest copies, induced by exchanging only two elements, cannot be PAC learned. 
\begin{dfn}
\label{def:absolutelynonlearnable}
    A graph $G$ is \emph{absolutely non-learnable} if $\VC{\Fiso{2}{G}} = \infty$.
\end{dfn}

%Recall that a relational structure $S$ is \emph{homogeneous} if every isomorphism between finite substructures extends to an automorphism of the entire structure. Hence, automorphic triviality is an exceptionally strong form of homogeneity. This may lead to the impression that homogeneous structures should be susceptible to at least \emph{some} learning. Such an impression is wrong: as we will see below, the random graph is absolutely non-learnable. \commgio{inline}{Citare una fonte sulla def. di omogeneità e sul fatto che il grafo random è omogeneo. Inoltre: è vero che automorficamente triviale implica omogeneo?}Indeed, 
We show that a local form of the extension property of the random graph (captured by \Cref{def:almost-random} below) characterizes the absolutely non-learnable graphs: as a consequence, we obtain that the random graph is an example of an absolutely non-learnable graph. The reader is referred to the excellent survey \cite{Cameron} for the properties of the random graph.

\begin{dfn}\label{def:almost-random}
    A graph $G$ is \emph{almost random} if, for every $n$, there exists a set $A \subseteq V(G)$ of cardinality $n$ such that, for every partition of $A$ into two disjoint subsets $X$ and $Y$, one can find $z \in V(G)$ which is adjacent to all vertices in $X$ and to no vertex in $Y$.
    \begin{comment}
    \begin{multline*}
    (\forall n)(\exists A\subseteq V(G), |A|=n)(\forall X\sqcup Y=A)(\exists z \in V(G))
    (z \in \bigcap_{x\in X} Nb(x) \smallsetminus\bigcup_{y\in  Y} Nb(y)).
    \end{multline*}
    \end{comment}
\end{dfn}

\begin{thm}\label{thm:ar-equals-anl}
    A graph $G$ is absolutely non-learnable if and only if it is almost random.
\end{thm}
\begin{proof}
    Fix $n$, we want to show that there exists $A\subseteq V(G)$ such that $|A|=n$ and for every partition $X\sqcup Y=A$, there exists $z_{X,Y}\in V(G)$ such that $X\subseteq N(z_{X,Y})$ and $Y\cap N(z_{X,Y})=\emptyset$.
    
    Let $W=\{(u_i,v_i)\}_{i<4n(2^{n+1}+1)}$ be a set of edges witnessing that $VC(\Fiso{2}{G})\geq 4n(2^{n+1}+1)$. Let $G_0 \in \Fiso{2}{G}$ be the graph with no edges from $W$, and $G_1 \in \Fiso{2}{G}$ be the graph with all edges form $W$.
    
%\commgio{}{the following is copy-pasted from 5.2, change the proof of that one since now arguably fewer details are needed.}
Notice that $G_0$ was obtained from $G$ by moving at most $2$ edges, and the same holds for $G_1$: hence, $G_0$ can be obtained from $G_1$ moving at most $4$ vertices. It follows that there are $4$ vertices $\{x_0,\dots x_{3}\}$ such that for all $i < 4n(2^{n+1}+1)$ there is an $\ell< 4$ such that $x_\ell=u_i$ or $x_\ell=v_i$. For simplicity, we suppose $x_\ell=u_i$ in all cases. It follows that there are $\overline{\ell} < 4$ and $H\subseteq \{0,\dots, 4n(2^{n+1}+1)\}$ of size $n(2^{n+1}+1)$ such that $\{(x_{\overline{\ell}},v_i)\}_{i\in H}\subseteq W$.

   We divide $H$ in $2^{n+1}+1$ parts, say $H_0,\dots,H_{2^{n+1}}$, each of size  $n$. For every $j<2^{n+1}+1$, let $K_j$ the set of edges $\{(x_{\overline{\ell}},v_i)\}_{i\in H_j}$. For every $j<2^{n+1}+1$, we rename the vertices of $K_j$ in such a way that $K_j=\{(x_{\overline{\ell}}, v^j_k)\}_{k<n}$.

    There are exactly $2^n$ graphs $S$ such that $V(S)=\{0,\dots,n\}$ and $E(S)\subseteq \{(n,1),\dots, (n,n-1)\}$. Call them $S_0,\dots, S_{2^n-1}$ and let $K^i$ be the graph such that
    \[
    V(K^i)=H \cup \{x_{\overline{l}}\} \, \text{ and } \, E(K^i)=\{(x_{\overline{\ell}}, v^j_k ) : (n,k)\in E(S_i),\ j<2^{n+1}+1 \text{ and } k<n\}
    \]
    Essentially, we are copying the structure of $S_i$ on the $K_j$'s.

    By the assumption on $W$, every $K^i$ can be obtained from $G$ moving at most $2$ vertices. Hence, we can obtain any of the $K^j$ from $G$ moving at most $2\cdot 2^n$ vertices. It follows that there is at least one $\overline{j}< 2^{n+1}+1$ such that none of the vertices $\{v^{\overline{j}}_0,\dots, v^{\overline{j}}_{n-1}\}$ are being moved. It follows that $x_{\overline{\ell}}$ needs to be moved to realize any of the $K^i$: it is thus easy to see that the points to which $x_{\overline{\ell}}$ is mapped are the $z_{X,Y}$ that we are looking for. This shows that if $\VC{\Fiso{2}{G}}$ is infinite, then $G$ is almost random.

    Conversely, assume that $G$ is almost random and let $n \in \N$: by definition, there are $v_1, \dots, v_n$ such that, for every $S \subseteq \{1, \dots, n\}$, there is $z_S$ for which $(z_S, v_i) \in E(G)$ if and only if $i \in S$. For every $S \subseteq \{1, \dots, n\}$, let $h_S \in \Fiso{2}{G}$ such that $h_S(z_{\emptyset})=z_S$: then the $h_S$'s witness that $\Fiso{2}{G}$ shatters $(z_{\emptyset}, v_1), \dots, (z_{\emptyset},v_n)$.
\end{proof}

%\begin{proof} Assume that $G$ is absolutely non-learnable and let $n \in \N$. Observe that, for every graph $G$, if $(u,v), (u',v') \in \N\times \N$ are disjoint (i.e.~$\{u,v\} \cap \{u',v'\} = \emptyset$), then $\Fiso{2}{G}$ does not shatter $(u,v), (u',v')$. %Indeed, starting from a copy $G$ in which $\chi_G (u,v)=i$ and $\chi_G (u',v')=i'$, no permutation with support of size 2 generates a copy $G'$ with $\chi_G (u,v)=1-i$ and $\chi_G (u',v')=1-i'$. 
%Therefore, $\Fiso{2}{G}$ must shatter a tuple of pair of vertices of the form $(z,v_1), \dots, (z,v_n)$. Without loss of generality, let $G$ be the element in $\Fiso{2}{G}$ such that $(z,v_i) \in E(G)$ for all $1 \le 1 \le n$. Since $\Fiso{2}{G}$ shatters $(z,v_1), \dots, (z,v_n)$, for every $S \subseteq \{1, \dots, n\}$ there is a permutation $h_S$ such that $h_S(z)=z_S$ and $(z_S,v_i) \in h_S(G)$ if and only if $i \in S$. Thus, the $z_S$'s together with $A=\{v_1, \dots, v_n\}$ satisfy the definition of almost random graph. 

%Conversely, assume that $G$ is almost random and let $n \in \N$: by definition, there are $v_1, \dots, v_n$ such that, for every $S \subseteq \{1, \dots, n\}$, there is $z_S$ for which $(z_S, v_i) \in E(G)$ if and only if $i \in S$. For every $S \subseteq \{1, \dots, n\}$, let $h_S \in \Fiso{2}{G}$ such that $h_S(z_{\emptyset})=z_S$: then the $h_S$'s witness that $\Fiso{2}{G}$ shatters $(z_{\emptyset}, v_1), \dots, (z_{\emptyset},v_n)$.
%\end{proof}

The promised absolute non-learnability of the random graph follows as an immediate corollary (since, obviously, the random graph is, in particular, almost random).
\begin{cor}
\label{cor:randomgraph}
    The random graph $\mathcal{R}$ is absolutely non-learnable.
\end{cor}

\section{Learning presentations with bounded support}
\label{sec:finitesupport}
%In \Cref{sec:online-learnable}, we have seen that the whole isomorphism type of a graph is online learnable precisely when the graph is automorphically trivial. Additionally, in \Cref{sec:abs-non-learnable}, we have characterized those graphs $G$ for which we cannot learn $\Fiso{2}{G}$ even in the PAC sense by means of the notion of almost randomness. 
The results of the previous sections impliy that, for graphs $G$ which are neither automorphically trivial nor almost random, we should be capable of learning $\Fiso{k}{G}$, at least for some $k$. Since the results in \Cref{sec:online-learnable} imply that automorphically trivial graphs are precisely those whose entire class of copies induced by permutations with bounded support is learnable, the immediately weaker notions worth considering are the following. 
\begin{dfn}\label{def:weakly-learnable}
    A graph $G$ is \emph{weakly online} (respectively, \emph{PAC}) \emph{learnable} if, for every $k \in \N$, $\Fiso{k}{G}$ is online (respectively, PAC) learnable 
\end{dfn}

%Notice that in the above definition we only require that the Littlestone dimension or the VC dimension of each of the $\Fiso{k}{G}$'s is finite, but their sequence does not need to be bounded. Thus, every online learnable graph is, in particular, weakly online learnable, but the converse might not necessarily hold: indeed, in \Cref{sec:separation}, we will exhibit a class of graphs which are weakly online learnable but not online learnable. 

In \Cref{section:fisok} we show that, perhaps surprisingly, weakly online (respectively, PAC) learnability of a graph $G$ boils down to online (PAC) learnability of $\Fiso{2}{G}$. As a consequence, we obtain that any graph $G$ is contained in one of the classes depicted in \Cref{fig:landscape}.
So far, we have provided examples of online learnable graphs and absolutely non-learnable graphs: \Cref{sec:separation} will be devoted to separate the notion of weak online learnability from online learnable, and that of weak PAC learnability from weak online learnability.

\subsection{Learning $\Fiso{k}{G}$ boils down to learning $\Fiso{2}{G}$} \label{section:fisok}
We begin by proving that, if $\VC{\Fiso{k}{G}}$ is large enough, then $\VC{\Fiso{2}{G}}$ must be large, too.

We find it convenient to introduce the following (rather natural) notation: given a graph $G$ and a set of pairs $H = \{(u_i,v_i)\}_{i\in I} \subseteq \N\times \N$, a \emph{configuration of $H$} is any binary string $\tau \in \{0,1\}^I$. We say that a configuration $\tau$ \emph{is realized in $\Fiso{k}{G}$} if there is a graph $G'\in \Fiso{k}{G}$ such that, for every $i\in I$, $(u_i,v_i)\in E(G')$ if and only if $\tau(i)=1$.

\begin{lem}
\label{lem:vc-iso-collapse}
    For every graph $G$ and $j\in\N$, $\VC{\Fiso{k}{G}}\geq 2(k+1)kj$ implies $\VC{\Fiso{2}{G}}\geq j$.
%\commgio{}{General comment for both proofs: explain why 2k; present construction as a proof by contradiction}
\end{lem}
%\begin{proof}
%   Let $W=\{(u_i,v_i)\}_{i\in 2(k+1)kj}$ be a set of edges witnessing that $\VC{\Fiso{k}{G}}\geq 2(k+1)kj$. Let $G_0\in \Fiso{k}{G}$ be the graph with no edge from $W$, and $G_1\in\Fiso{k}{G}$ be the graph with all the edges from $W$: since we changed the status of all the edges moving at most $2k$ points, it follows that there are $2k$ points $\{x_0,\dots x_{2k-1}\}$ such that for all $i\in 2(k+1)kj$ there is an $\ell\in 2k$ such that $x_\ell=u_i$ or $x_\ell=v_i$. For simplicity, we suppose $x_\ell=u_i$ in all cases. It follows that there are $\overline{\ell}\in 2k$ and $H\subseteq 2(k+1)kj$ of size at least $(k+1)j$ such that $\{(x_{\overline{\ell}},v_i)\}_{i\in H}\subseteq W$.

%   We divide $H$ in $k+1$ parts, each of size at least $j$, and we call them $H_0,\dots,H_k$. If we can reach every configuration of $\{(x_{\overline{\ell}},v_i)\}_{i\in H_0}$ just by moving $x_{\overline{\ell}}$ with respect to the configuration in $G$, then $\{(x_{\overline{\ell}},v_i)\}_{i\in H_0}$ witnesses that $\VC{\Fiso{2}{G}}\geq j$. So suppose this is not the case: there is a configuration of $\{(x_{\overline{\ell}},v_i)\}_{i\in H_0}$, which we call $G^0$, that requires moving one of the $v_i$. We argue like this for all the $k+1$ components: if for every $h\in k+1$ we find a configuration $G^h$ that requires moving one of the $v_i$ with $i\in H_h$, then the configuration of $\{(x_{\overline{\ell}},v_i)\}_{i\in H}$ resulting from pasting together the $G^h$ would require moving (at least) $k+1$ points, contradiction.
%\end{proof}

\begin{proof}
   Let $W=\{(u_i,v_i)\}_{i< 2(k+1)kj}$ be a set of edges witnessing that $\VC{\Fiso{k}{G}}\geq 2(k+1)kj$. Let $G_0\in \Fiso{k}{G}$ be the graph with no edges from $W$, and $G_1\in\Fiso{k}{G}$ be the graph with all the edges from $W$. Notice that $G_0$ was obtained from $G$ by moving at most $k$ edges, and the same holds for $G_1$: hence, $G_0$ can be obtained from $G_1$ moving at most $2k$ vertices. It follows that there are $2k$ vertices $\{x_0,\dots x_{2k-1}\}$ such that for all $i < 2(k+1)kj$ there is an $\ell< 2k$ such that $x_\ell=u_i$ or $x_\ell=v_i$. For simplicity, we suppose $x_\ell=u_i$ in all cases. It follows that there are $\overline{\ell} < 2k$ and $H\subseteq \{0,\dots,2(k+1)kj-1\}$ of size at least $(k+1)j$ such that $\{(x_{\overline{\ell}},v_i)\}_{i\in H}\subseteq W$.

    Let us assume for a contradiction that $\VC{\Fiso{2}{G}}<j$. 
   We divide $H$ in $k+1$ parts, each of size at least $j$, and we call them $H_0,\dots,H_k$. 
   We start by focusing on $H_0$. Notice that we cannot realize every configuration of $\{(x_{\overline{\ell}},v_i)\}_{i\in H_0}$ in $\Fiso{k}{G}$ just by moving $x_{\overline{\ell}}$ with respect to the original assignment in $G$, since this would imply that $\VC{\Fiso{2}{G}}\geq j$. Hence, there is a configuration of $\{(x_{\overline{\ell}},v_i)\}_{i\in H_0}$, which we call $G^0$, that requires moving one of the $v_i$ to be realized in $\Fiso{k}{G}$. We argue like this for all the $k+1$ components, thus finding, for every $h<k+1$, a configuration $G^h$ that requires moving one of the $v_i$ with $i\in H_h$ to be realized. Then, the configuration of $\{(x_{\overline{\ell}},v_i)\}_{i\in H}$ resulting from pasting together the $G^h$ would require moving (at least) $k+1$ vertices to be realized in $\Fiso{k}{G}$. This gives the desired contradiction.
\end{proof}

%\commgio{inline}{the bound can obviously be improved, the 2 is not necessary. Also, the conclusion should be $\geq j+1$.}

%It follows that a graph is absolutely non-learnable if and only if $\VC{\Fiso{k}{G}}=\infty$ for any $k\in\omega\setminus 2$.
As an immediate consequence, we get the promised equivalent condition for weak PAC learnability.
\begin{thm}
\label{thm:2kPAC}
    A graph $G$ is weakly PAC learnable if and only if $\Fiso{2}{G}$ is PAC learnable.
\end{thm}

Similarly, if $\Fiso{k}{G}$ contains enough thresholds, then $\Fiso{2}{G}$ must contain a large number of thresholds. 
\begin{lem}
\label{lem:ld-iso-collapse}
    For every graph $G$ and $j\in\N$, if $\Fiso{k}{G}$ contains at least $2k(j+1)^{k+1}+1$ thresholds, then $\Fiso{2}{G}$ contains at least $j+1$ thresholds.
\end{lem}
\begin{proof}
    The first part of the proof is similar to the one of \Cref{lem:vc-iso-collapse}. Let $W=\{ (u_i,v_i)\}_{i< 2k(j+1)^{k+1}}$ be a set of edges witnessing that $\Fiso{k}{G}$ has $2k(j+1)^{k+1}+1$ thresholds. As before, let $G_0$ be the graph with no edges from $W$, and $G_1$ be the one with all those edges: by the same argument, we conclude that there is a special vertex $x_{\overline{\ell}}$ and $H \subseteq \{0,\dots, 2k(j+1)^{k+1}-1\}$ of size \emph{exactly} $(j+1)^{k+1}$ such that $\{ (x_{\overline{\ell}},v_i)\}_{i\in H}\subseteq W$ (it will be clear in the following why it is more practical to be more specific about the size of $H$ this time). 

    The second part of the proof is, in spirit, also close to the one of \Cref{lem:vc-iso-collapse}: the difference is that, whereas before we could find a bad configuration working independently on the various pieces that would compose it, here there is a strong dependency between the pieces. As before, we assume for a contradiction that $\Fiso{2}{G}<j+1$.
    
    To keep the notation manageable, we rename the numbers in $H$ so that
    they form the interval $[0, (j+1)^{k+1})$, of course preserving the order between the elements. Consider the set of edges $H_0=\{ (x_{\overline{\ell}},v_i): i\neq 0 \text{ and } i \text{ is divisible by } (j+1)^{k} \}$: there are $j$ such edges, and hence they can be seen as the domain of $j+1$ thresholds in $\Fiso{k}{G}$. 
    To be more specific: since we know that $W$ witnesses that $\Fiso{k}{G}$ contains a set of $2k(j+1)^{k+1}+1$ thresholds $T=\{h_0,\dots, h_{2k(j+1)^{k+1}}\}$, and $H$ is a subset of $W$ of which we have required to maintain the same order, for every $j\in H$ we can find $h_p\in T$ such that $h_p((x_{\overline{\ell}},v_i))=0$ if $i<j$ and $h_p((x_{\overline{\ell}},v_i))=1$ otherwise, and a threshold $h_q$ such that $h_q((x_{\overline{\ell}},v_i))=0$ for all $i\in H$. Hence, by restricting the domain of the thresholds that we have found this way, we can see $H$ as their domain. Similarly, we can restrict from $H$ to $H_0$ with the same considerations.
    By our contradictory assumption, we cannot realize all these thresholds just by moving $x_{\overline{\ell}}$. Then there is a configuration 
    %\commgio{}{Call it threshold, say it is realized instead of reached}
    $G^0$ of $H_0$ that is a threshold and requires moving at least one of the $v_i$ to be realized in $\Fiso{k}{G}$. In the interest of brevity, we write $G^0(x_{\overline{\ell}},v_i)=1$ to mean that $(x_{\overline{\ell}},v_i)\in E(G^0)$, and $G^0(x_{\overline{\ell}},v_i)=0$ to mean that $(x_{\overline{\ell}},v_i)\notin E(G^0)$. We define $i_0$ as
    \[
        \max \{0\}\cup \{ i\in\N: (x_{\overline{\ell}},v_i)\in H_0 \text{ and }G^0(x_{\overline{\ell}},v_i)=0 \}.
    \] 
    The intuition is clearly that $i_0$ identifies the vertex in which $G^0$ evaluated on $H_0$ starts having edges. The fact that $G^0$ could be the graph having all the edges is the reason why we have to use $v_0$ somewhat differently than the other vertices: since $\{ i\in\N: (x_{\overline{\ell}},v_i)\in H_0 \text{ and }G^0(x_{\overline{\ell}},v_i)=0 \}$ could be the empty set, adding $0$ to is makes sure that the max exists.

    Intuitively, we now use the same procedure between $v_{i_0}$ and $v_{i_0+1}$. More formally, let $H_1= \{ (x_{\overline{\ell}},v_i): i\in (i_0, i_0+(j+1)^k) \text{ and } i \text{ is divisible by } (j+1)^{k-1} \}$: these is again a set of $j$ edges, and so the support of $j+1$ thresholds. Again by our contradictory assumption, we can suppose that there is a configuration $G^1$ that requires moving one of the $v_i$ to be realized. The nature of this configuration will allow us to define $i_1$ analogously as what we did above, and so the procedure continues.

    Let us iterate the procedure $k+1$ times: notice then that we can combine the configurations $G_i$ obtained at the various step of the construction into $G'=G^0\cup G^1\cup \dots \cup G^k$ on $H_0\cup H_1\cup \dots \cup H_k$, with $G'\in\Fiso{k}{G}$. By the property we were requiring of the $G^i$s', realizing $G'$ in $\Fiso{k}{G}$ requires moving $k+1$ vertices, again a contradiction. 
\end{proof}

Combining \Cref{fct:thresholds} with \Cref{lem:ld-iso-collapse}, we obtained the following equivalence.
\begin{thm}
\label{thm:2konline}
    A graph $G$ is weakly online learnable if and only if $\Fiso{2}{G}$ is online learnable.
\end{thm}
%\commgio{inline}{Again, the bounds are most likely suboptimal but allow me to be lazy when writing the proof.}

%\begin{cor}
    %For every graph $G$, if $\Ld{\Fiso{k}{G}}\geq 2^{2k(j+1)^{k+1}}$, then $\Ld{\Fiso{2}{G}} \geq \log (j+1)$.
%\end{cor}

\subsection{Separating the learning classes}\label{sec:separation}
%In the previous pages, we characterized which graphs are absolutely (non) learnable and those for which $\Fiso{k}{G}$ is online, and PAC learnable. Section \Cref{section:fisok} proves that $\Fiso{k}{G}$ is online (respectively, PAC) learnable if and only if $\Fiso{2}{G}$ is online (respectively, PAC) learnable.
It is clear from \Cref{def:weakly-learnable} that any online learnable graph is, in particular, weakly online learnable. 
We begin this section by showing that the converse does not hold. Indeed, there is a natural class of graphs, denoted by $\mathfrak{K}_{equiv}$ witnessing the separation between these two notions. This class consists of all non-automorphically trivial graphs that are disconnected unions of cliques. Note that a disconnected union of cliques is non-automorphically trivial as long as it contains infinitely many cliques of size strictly greater than 1. As the notation $\mathfrak{K}_{equiv}$ suggests, each graph in this class can be viewed as an equivalence structure consisting of countably many equivalence classes, each of at most countable size, with infinitely many classes having size strictly greater than 1.
\begin{prop}
\label{prop:kequiv}
     If $G \in \mathfrak{K}_{\text{equiv}}$, then $G$ is weakly online learnable, but not online learnable. 
\end{prop}
\begin{proof} \Cref{fact:automorphicallytrivialgraphs} ensures that $G$ is not automorphically trivial and hence, by the results in \Cref{sec:online-learnable}, that is not online learnable. By \Cref{thm:2konline}, to conclude the proof, it suffices to show that $\Fiso{2}{G}$ is online learnable. 

We associate to any $v \in \N$ two counters $c_0^v$ and $c_1^v$ that are initialized at $0$ and, whenever $\mathbf{L}((v,w))=i$ and makes a mistake, we increase $c_i^x$ for $x \in \{v,w\}$ by $1$. We define our learner $\mathbf{L}$ letting $\mathbf{L}((v,w))=1$ if and only if $(v,w) \in E(G)$ and $c_i^x<2$ for  $i \in \{0,1\}$ and $x \in \{v,w\}$. Let $v,w \in \N$ and $i \in \{0,1\}$ such that, when $\mathbf{L}((v,w))=i$, then $c_i^v$ reaches 2. Then:
\begin{itemize}
\item  if $i=0$, then for any $x \in \N$, let $\mathbf{L}((v,x))=1$ if and only if $(w,x) \in E(G)$;
\item if $i=1$ then, we have two cases. If  $c_0^v=0$, for any $x \in \N$, let $\mathbf{L}((v,x))=0$. Otherwise,  if $c_0^v=1$, let $(v,z)$ be the pair witnessing $c_{0}^v=1$. For any $x \in \N$ let $\mathbf{L}((v,x))=1$ if and only if $(z,x) \in E(G)$.
\end{itemize}
%We now show that $\mathbf{L}$ can make at most $6$ mistakes. 
Let $a$ and $b$ be the vertices which have been permuted in the target copy. First, notice that, if $\mathbf{L}$ makes an error on a pair $(v,w)$, then $\{v,w\} \cap \{a,b\} \neq \emptyset$: thus, to bound the number of mistakes of $\mathbf{L}$, it suffices to consider the counters $c_0^a, c_1^a, c_0^b$ and $c_1^b$. Moreover, by definition of $\mathbf{L}$, if either $c_0^a=2$ or $c_0^a=1$ and $c_1^a=2$, then $\mathbf{L}$ can only make a further mistake, namely on the pair $(a,b)$. A symmetric argument works for $b$. Assume that we have already made $5$ mistakes. Then necessarily, for some $x \in \{a,b\}$, either $c_0^x=2$ or $c_0^x=1$ and $c_1^x=2$, meaning that the only possible mistake $\mathbf{L}$ can make at this point is on the pair $(a,b)$. Thus, $\mathbf{L}$ makes at most $6$ mistakes.
%To prove that $\mathb{L}$ online learns $\Iso{2}{G}$ notice that after We sketch the proof that $\mathbf{L}$ makes at most $4$ mistakes.First, it is immediate that $\mathbf{L}$ may make mistakes only on those pairs of vertices involving either $a$ or $b$ where $a$ and $b$ are the (only) vertices that have been permuted. Take the first four pairs of vertices on which $\mathbf{L}$ makes a mistake (if they do not exist we are done). By what we just said, $a,b$ must be contained in these pairs of vertices and from what we said above at least one of them must appear twice. The following cases are considered:
%\begin{itemize}
%\item Suppose that $a$ appears twice, and on such pairs $\mathbf{L}$ makes mistakes answering $0$.
%\begin{itemize}
%\item if $a$ appears three times, then the only other possible mistake is the one on the pair $(a,b)$ and we are done.
%\item if $a$ appears twice then also $b$ appears twice and hence
%\begin{itemize}
%\item if $\mathbf{L}$ makes one mistake answering $0$, by definition, $\mathbf{L}$ starts to answer only $0$. The only other mistake that can occur is hence answering $0$
%\end{itemize}
%\end{itemize}
%if $c_0^a=c_1^a=2$ it is immediate that that also $b$ is in such pairs
%\item if $c_a^v=c_i^b=2$ it is clear
%\item if $c_0^a=2$ and $c_0^b<2$ then 
%\end{itemize}
\end{proof}

We conclude this section by showing that, contrary to the case of learning the whole isomorphism type, where PAC and online learnability coincide, weak online learning is, in general, more demanding than weak PAC learning. In order to do so, we first give the following characterization of online learnability for $\Fiso{2}{G}$.%, whose easy proof is given in \Cref{app:iso2-not-online}.

\begin{lem}\label{lem:when-iso2-not-online}
    For any graph $G$, $\Fiso{2}{G}$ is not online learnable if and only if for every $n$, there exist $u_0, \dots, u_n, v_1, \dots, v_n \in V(G)$ such that $ v_j \in N(u_i)$ if and only if $j \le i$.
\end{lem}
\begin{proof}

 Suppose $\Fiso{2}{G}$ is not online learnable. Fix $n$, and let $W=\{(u_i,v_i)\}_{i<4n(2n+3)}$ be a set of edges witnessing that $\Fiso{2}{G}$ has $ 4n(2n+3)+1$ thresholds. Let $G_0 \in \Fiso{2}{G}$ be the graph with no edges from $W$, and $G_1 \in \Fiso{2}{G}$ be the graph with all edges form $W$.

    For a similar reasoning as the one in \Cref{thm:ar-equals-anl},  $G_0$ can be obtained from $G_1$ moving at most $4$ vertices. It follows that there are $4$ vertices $\{x_0,\dots x_{3}\}$ such that for all $i < 4n(2n+3)$ there is an $\ell< 4$ such that $x_\ell=u_i$ or $x_\ell=v_i$. For simplicity, we suppose $x_\ell=u_i$ in all cases. It follows that there are $\overline{\ell} < 4$ and $H\subseteq \{0,\dots,4n(2n+3)-1\}$ of size $n(2n+3)$ such that $\{(x_{\overline{\ell}},v_i)\}_{i\in H}\subseteq W$.

To keep the notation manageable, we rename the numbers in $H$ so that
    they form the interval $[0, n(2n+3))$, of course preserving the order between the elements.

   We divide $H$ in $n$ parts, each of size  $2n+3$ and we call them $H_0,\dots,H_{n-1}$: specifically, for every $j<n$, $H_j=[j(2n+3),(j+1)(2n+3))$. For every $j<n$ we call $K_j$ the set of edges $\{(x_{\overline{\ell}},v_i)\}_{i\in H_j}$. For every $j<n$, we rename the vertices of $K_j$ in such a way that $K_j=\{(x_{\overline{\ell}}, v^j_k)\}_{k<n}$.

    For every $i< n+1$, let $K^i$ be the graph such that
    \[
    V(K^i)=H \cup \{x_{\overline{l}}\} \, \text{ and } \, E(K^i)=\{(x_{\overline{\ell}}, v^j_k ) : j\geq i\}
    \]
    In essence, $K^i$ realizes the threshold $h_i$ such that $(x_{\overline{\ell}}, v_p)\in E(h_i)$ if and only if $p<i(2n+3)$.  
    
    By the assumption on $W$, every $K^i$ can be obtained from $G$ moving at most $2$ vertices. Hence, we can obtain any of the $K^i$ from $G$ moving at most $2n$ vertices. It follows that every $H_j$ contains a vertex that does not need to be moved to realize any of the $K^i$'s. In particular, this means that $x_{\overline{\ell}}$ needs to be moved to realize any of the $K^i$. Let $v_{j+1}$ be such a vertex in $H_j$. It is then clear that $u_0,\dots,u_n$ is given by the images of $x_{\overline{\ell}}$. 

    Conversely, assume that, for a given $n$, $u_0, \dots, u_n, v_1, \dots, v_n \in \N$ are such that $(u_i, v_j) \in E(G)$ if and only if $j \le i$. For every $i$, let $h_i$ be the graph obtained from $G$ swapping $u_0$ and $u_i$, and acting as the identity on all other vertices: then $h_0(G), \dots, h_n(G) \in \Fiso{2}{G}$ and $v_1, \dots, v_n \in \N$ witness that $\Fiso{2}{G}$ contains $n$ thresholds
    \end{proof}

    %Conversely, assume that $G$ is almost random and let $n \in \N$: by definition, there are $v_1, \dots, v_n$ such that, for every $S \subseteq \{1, \dots, n\}$, there is $z_S$ for which $(z_S, v_i) \in E(G)$ if and only if $i \in S$. For every $S \subseteq \{1, \dots, n\}$, let $h_S \in \Fiso{2}{G}$ such that $h_S(z_{\emptyset})=z_S$: then the $h_S$'s witness that $\Fiso{2}{G}$ shatters $(z_{\emptyset}, v_1), \dots, (z_{\emptyset},v_n)$.

%\begin{proof}%[Proof of \Cref{lem:when-iso2-not-online}] 
%Assume that $\Fiso{2}{G}$ is not online learnable. Then, by \Cref{fct:thresholds}, $\Fiso{2}{G}$ contains $n$ thresholds for any $n$. Observe that $\Fiso{2}{G}$ cannot realize both thresholds $00$ and $11$ on two disjoint pairs of vertices: thus, these $n$ thresholds must be realized by $h_0, \dots, h_n$ a tuple of the form $(u,v_1), \dots, (u,v_n)$. Then $h_0(u), \dots, h_n(u), v_1, \dots, v_n$ are vertices as claimed. 

%Conversely, assume that, for a given $n$, $u_0, \dots, u_n, v_1, \dots, v_n \in \N$ are such that $(u_i, v_j) \in E(G)$ if and only if $j \le i$. For every $i$, let $h_i$ be the permutation swapping $u_0$ and $u_i$, and acting as the identity on all other vertices: then $h_0, \dots, h_n \in \Fiso{2}{G}$ and $v_1, \dots, v_n \in \N$ witness that $\Fiso{2}{G}$ contains $n$ thresholds. 
%\end{proof} 

The announced separation is witnessed by the graph 
$\mathsf{R} = (\N, \{ (2i,2j+1) \colon i, j \in \N \ \text{and} \ i \le j \})$,
depicted below.
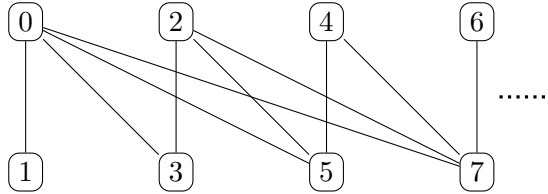
\begin{center}
    \begin{tikzpicture}
        \node[rectangle, rounded corners, draw=black, text centered] (0) at (0,0) {0};
        \node[rectangle, rounded corners, draw=black, text centered] (2) at (2,0) {2};
        \node[rectangle, rounded corners, draw=black, text centered] (4) at (4,0) {4};
        \node[rectangle, rounded corners, draw=black, text centered] (6) at (6,0) {6};
        %\node[rectangle, rounded corners, draw=black, text centered] (8) at (8,0) {8};

        \node[rectangle, rounded corners, draw=black, text centered] (1) at (0,-2) {1};
        \node[rectangle, rounded corners, draw=black, text centered] (3) at (2,-2) {3};
        \node[rectangle, rounded corners, draw=black, text centered] (5) at (4,-2) {5};
        \node[rectangle, rounded corners, draw=black, text centered] (7) at (6,-2) {7};
        %\node[rectangle, rounded corners, draw=black, text centered] (9) at (8,2) {9};
        \draw[very thick, dotted] (6.3,-1) -- (7,-1);

        \draw (0) -- (1);
        \draw (0) -- (3);
        \draw (2) -- (3);
        \draw (0) -- (5);
        \draw (2) -- (5);
        \draw (4) -- (5);
        \draw (0) -- (7);
        \draw (2) -- (7);
        \draw (4) -- (7);
        \draw (6) -- (7);
        
    \end{tikzpicture}
\end{center}

\begin{prop}
\label{prop:onlinepacdifferent}
The graph $\mathsf{R}$ is weakly PAC learnable but not weakly online learnable. 
\end{prop}
\begin{proof}
The fact that $\Fiso{2}{G}$ is PAC learnable follows combining the fact that $G$ is not almost random and \Cref{thm:ar-equals-anl}. Moreover, it is easy to see that $G$ satisfies the property of \Cref{lem:when-iso2-not-online}, meaning that $\Fiso{2}{G}$ is not online learnable.
\end{proof}

\section{Complexity of different learning notions}
\label{sec:complexity}
In this section, we provide a complete characterization of all classes of graphs that are learnable under any of the paradigms discussed so far. We begin by applying tools from descriptive set theory, specifically within the framework of Wadge reducibility; for a detailed account of this approach, see \cite{kechris2012classical}. Then, in Subsection~\ref{sec:complexitycomputable}, we extend our analysis using techniques from computability theory, allowing us to incorporate the case of computable graphs.

\subsection{Wadge reducibility}
\label{sec:complexityprel}
Let $X$ be a \emph{Polish space} (i.e., a separable completely metrizable topological space) and let $\mathbf{\Sigma}^0_1(X)$ denote the family of open subsets of $X$. For $n \geq 1$, the classes $\mathbf{\Sigma}^0_n(X)$ and $\mathbf{\Pi}^0_n(X)$ of subsets of $X$ are defined by recursion: $\mathbf{\Pi}^0_n(X)$ is the class of all complements of sets in $\mathbf{\Sigma}^0_n(X)$, while for $n \geq 2$, $\mathbf{\Sigma}^0_n(X)$ consists of all countable unions $\bigcup_{m \in \mathbb{N}} A_m$ where $A_m \in \mathbf{\Pi}^0_{k_m}(X)$ for some $m < n$. Notice that such a hierarchy is unbounded, i.e., $\mathbf{\Sigma}^0_n(X) \cup \mathbf{\Pi}^0_n(X) \subsetneq \mathbf{\Sigma}^0_m(X) \cap \mathbf{\Pi}^0_m(X)$ whenever $n < m$. Although we will not use this fact in the present paper, we note that this hierarchy can be extended to any ordinal $\alpha< \omega_1$, where $\omega_1$ is the first uncountable ordinal; this extended hierarchy is known as the \emph{Borel hierarchy}.

In this paper, we can always assume that $X$ is the Cantor space $2^{\mathbb{N}}$, i.e., the space of all infinite binary sequences with the product topology (where the base space $\{0,1\}$ is equipped with the discrete topology). It is well-known that such a space is Polish.

Let $\mathbf{\Gamma}$ be a class of sets in Polish spaces and $\check{\mathbf{\Gamma}}$ its dual class, i.e., the family of complements of elements of $\mathbf{\Gamma}$. The class $D_2(\mathbf{\Gamma})$ is the family of sets of the form
$A \smallsetminus B$
where $A, B \in \mathbf{\Gamma}$. This is equivalent to the class of all intersections $A \cap B$ with $A \in \mathbf{\Gamma}$ and $B \in \check{\mathbf{\Gamma}}$. As the subscript $2$ indicates, this notion can be further extended (even transfinitely), yielding the \emph{difference hierarchy} relative to $\mathbf{\Gamma}$.

Wadge reducibility provides a notion of complexity for sets in Polish spaces. Given two Polish spaces $X$ and $Y$, and sets $A \subseteq X$, $B \subseteq Y$, we say that $A$ is \emph{Wadge reducible} to $B$ if there exists a continuous function $f: X \rightarrow Y$ such that
$x \in A \iff f(x) \in B$.
Intuitively, this means that $A$ is no more complex than $B$. Wadge reducibility defines a quasi-order whose equivalence classes are called the \emph{Wadge degrees}. Let $\mathbf{\Gamma}$ be a boldface class, and let $X$ and $Y$ be Polish spaces with $X$ zero-dimensional (notice that the Cantor space is zero-dimensional). A set $B \subseteq Y$ is $\mathbf{\Gamma}$-\emph{hard} if for every $A \in \mathbf{\Gamma}(X)$, one has that $A$ is Wadge reducible to $B$; if in addition $B \in \mathbf{\Gamma}(Y)$, then $B$ is $\mathbf{\Gamma}$-\emph{complete}. Note that if $B$ is $\mathbf{\Gamma}$-hard, then its complement is $\check{\mathbf{\Gamma}}$-hard. Also, if $B$ is $\mathbf{\Gamma}$-hard and $B$ is Wadge reducible to $A$, then $A$ is $\mathbf{\Gamma}$-hard as well. These statements remain true if we replace hardness with completeness. This gives a useful technique to prove that a set $B$ is $\mathbf{\Gamma}$-hard (or $\mathbf{\Gamma}$-complete$)$: find a known $\mathbf{\Gamma}$-hard (or $\mathbf{\Gamma}$-complete) set $A$ and show that $A$ is Wadge reducible to $B$.

To obtain hardness results we use some well-known benchmark sets from the literature (see \cite[Section 23.A]{kechris2012classical}). 
\begin{lem}
\label{lem:canonicalcompletesets}
The following holds:
\begin{itemize}
\item The set $\mathsf{EC}:=\{p \in 2^\N : (\forall^\infty n)(p(n)=0)\}$ is $\mathbf{\Sigma}_2^0$-complete and its complement $\mathsf{NEC}:=\{p \in 2^\N : (\exists^\infty n)(p(n)=1)\}$ is $\mathbf{\Pi}_2^0$-complete;
\item The set $\mathsf{EC} \times \mathsf{NEC}:=\{(p,q) \in 2^\N \times2^\N : (\forall^\infty n)(p(n)=0) \land (\exists^\infty n)(q(n)=1)\}$ is $D_2(\mathbf{\Sigma}_2^0)$-complete.
\end{itemize}
\end{lem}
The complexity of set $\mathsf{EC}$ and its complement is well-established (\cite[Exercise 23.1]{kechris2012classical}). On the other hand the complexity of $\mathsf{EC} \times \mathsf{NEC}$ follows in a straightforward way from the hardness of $\mathsf{EC}$ and its complement and the fact that it is defined on the product space $2^\N \times 2^\N$.
\subsection{Wadge complexity of learnable classes of graphs}
We start giving the following easy lemma, whose proof is omitted as it directly follows from the definitions of the learning paradigms we have defined so far. Notice that, for PAC learnability, a stronger result has already been proved in \Cref{lem:VC-induced-subgraphs}.

\begin{lem}
\label{lem:almostrandomclosed}
    Let $G$ be not $\mathfrak{P}$-learnable where
    $\mathfrak{P} \in \{\text{online, PAC, weakly online, weakly PAC}\}$.
     Then, for every induced supergraph $H$ of $G$, $H$ is not $\mathfrak{P}$-learnable either. 
\end{lem}

Since $G$ is not weakly PAC learnable if and only if it is absolutely non learnable (\Cref{thm:ar-equals-anl}) it follows that the lemma above holds replacing \lq\lq not $\mathfrak{P}$-learnable\rq\rq\ with \lq\lq absolutley non learnable\rq\rq.

\begin{thm}
    \label{thm:firstcomplexity}
The classes
\[\{G : G \text{ is an }\mathfrak{P}\text{-learnable graph}\}\]
where $\mathfrak{P} \in \{\text{online, PAC, weakly online, weakly PAC}\}$ are $\mathbf{\Sigma}_2^0$-complete while the class
\[\{G : G \text{ is an absolutely non-learnable graph}\}\]
 is $\mathbf{\Pi}_2^0$-complete.
\end{thm}
\begin{proof}
    First notice that the definitions of online learnable, weakly online learnable, and weakly PAC learnable graphs are $\mathbf{\Sigma}_2^0$ while the one of absolutely non-learnable graphs is $\mathbf{\Pi}_2^0$.

    We define a function $h: 2^{\mathbb{N}} \rightarrow 2^{\mathbb{N}}$ reducing $\mathsf{EC}:=\{p \in 2^\N: (\forall^\infty n)(p(n)=0)\}$ (which, by \Cref{lem:canonicalcompletesets}, is $\mathbf{\Sigma}_2^0$-complete) to the class of online learnable graphs and $\mathsf{NEC}$ (which, by \Cref{lem:canonicalcompletesets}, is $\mathbf{\Pi}_2^0$-complete) to the class of absolutely non-learnable graphs. Notice that the same reduction also shows that the class of weakly online learnable and weakly PAC learnable graphs is $\mathbf{\Sigma}_2^0$-complete. This follows from the our previous results, namely that:
  \begin{itemize}
\item  online learnable $\subset$ weakly online learnable $\subset$ weakly PAC learnable and
\item all classes mentioned in the item above are disjoint from the class of almost-random graphs.
  \end{itemize}

The announced reduction $h$ is defined as follows. Given in input $p \in 2^{\mathbb{N}}$, at stage $0$ do nothing. At stage $s+1$ if $p(s+1)=0$ add the minimum fresh vertex (i.e., a vertex for which we have not decided yet its membership to $h(p)[s]$) as an isolated vertex to $h(p)[s]$.
 If $p(s+1)=1$, for every (finite) subset $U$ of vertices in $h(p)[s]$ we add a fresh vertex $v$ joined to all vertices in $U$ but to none outside $U$. 

  To see that $h$ is the desired reduction notice that if $(\forall^\infty n)(p(n)=0)$ then $h(p)\cong \mathcal{R}[fin] \oplus K_\N$ where $\mathcal{R}[fin]$ is the restriction of $\mathcal{R}$ (the Random graph) to a finite set of vertices. Hence $h(p)$ is automorphically trivial and, by \Cref{thm:at-implies-online}, online learnable. On the other hand, if $(\exists^\infty n)(p(n)=1)$, $h(p)$ contains a copy of  $\mathcal{R}$ (witnessed by those stages $s$ such that $p(s)=1$). By  \Cref{cor:randomgraph}, \Cref{thm:ar-equals-anl} and \Cref{lem:almostrandomclosed} $h(p)$ is absolutely non-learnable and this concludes the proof.
\end{proof}
The last part of this section is devoted to prove the complexity of those sets which include all and only those graphs that are weakly online/PAC  learnable. Before doing so we need the following lemma showing that the notions of weakly online learnability and weakly PAC learnability are closed under disconnected union of graphs. Notice that this is not the case for online learnability. For example, $\overline{K_\N}$ and $K_\N$ are automorphically trivial (and hence, by \Cref{thm:at-implies-online}, online learnable), but $\overline{K_\N} \oplus K_\N$ is not.
\begin{lem}
\label{lem:disjointunionalmostrandom}
If $G$ and $H$ are weakly online learnable (respectively, weakly PAC learnable) then $G \oplus H$ is weakly online learnable (respectively, weakly PAC learnable).
\end{lem}
\begin{proof}
    %Both proofs follow the same pattern. 
    %We first prove the claim for weakly PAC learnability as the proof is slightly more involved.

    We focus first on the case of non-weakly PAC learnability.
    Notice that if a graph $F$ is weakly PAC learnable, by \Cref{def:almost-random}, it is not almost random. Hence, there is a minimum $n_F$ witnessing that $F$ is weakly PAC learnable, i.e., there is no $A_{n_F}\subseteq V(F)$ of cardinality $n_F$ for which for every partition of $A_{n_F}$ into two disjoint subsets $X$ and $Y$ there is $z_{X,Y} \in V(F)$ that is adjacent to all vertices in $X$ and to no vertex in $Y$. Notice that for every $m \geq n$, $m$ witnesses that $F$ is weakly PAC learnable either. 

    Suppose that $G$ and $H$ are weakly PAC learnable, we can suppose that $n_G\geq n_H$. We claim that $n_G+1$ witnesses that $G\oplus H$ is weakly PAC learnable. Suppose not and let $A\subseteq V(G\oplus H)$ of size $n_G+1$ be a counterexample. Notice that either $A\cap G$ or $A\cap H$ is empty, otherwise there should be a vertex connected to both $G$ and $H$.
    Hence, we can suppose that $A\subseteq G$ (since $n_G\geq n_H$, this is the only interesting case). Fix a vertex $a\in A$, then for every partition $X,Y$ of $A$ with $a\in X$, $z_{X,Y}\in V(G)$, contradicting the definition of $n_G$.

    The argument for non-online learnability is similar: by \Cref{lem:when-iso2-not-online}, online learnability of a graph $F$ is witnessed by a minimal $n_F$ such that for every choice of vertices $\{u_0,\dots,u_{n_F},v_1,\dots v_{n_F}\}$, it is not true that $v_j\in N(u_i)$ if and only if $j\leq i$. Again assuming that $n_G\geq n_H$, the claim is that $n_G+1$ witnesses the online learnability of $G\oplus H$. Indeed, suppose this is not the case, and let $\{u_0,\dots,u_{n_G+1},v_1,\dots v_{n_G+1}\}$ be a witness. Notice that all the points other than $u_0$ have to be in the same connected component of $G\oplus H$, and this is enough to conclude by definition of $n_G$ and some easy considerations. 
\end{proof}

\begin{thm}
\label{thm:firstcomplexityD2}
     The class of graphs
     \[\{G : G \text{ is a weakly online learnable but not online learnable graph}\}\]
      is $D_2(\mathbf{\Sigma}_2^0)$-complete.
\end{thm}

\begin{proof}
    First notice that the definition of such a class is $D_2(\mathbf{\Sigma}_2^0)$.
 
    We will reduce the set $\mathsf{EC} \times \mathsf{NEC}:=\{(p,q) \in 2^\N \times2^\N : (\forall^\infty n)(p(n)=0) \land (\exists^\infty n)(q(n)=1)\}$ 
 (which, by \Cref{lem:canonicalcompletesets} is $D_2(\mathbf{\Sigma}_2^0)$-complete) to the class of graphs that are weakly online learnable but not online learnable.

Our desired reduction combines two functions $f,g:2^\N\rightarrow2^\N$ as follows. I.e.,  given in input $(p,q)$ our reduction outputs the disconnected union of $f(p)$ and $g(q)$.

The function $f$ is defined as follows. Given in input $p$, at stage $0$ do nothing. At stage $s+1$, if $p(s+1)=0$, add $2s$ and $2s+1$ as isolated vertices to  $f(p)[s]$: such vertices will be marked as \emph{unused}. If $p(s+1)=1$, then add $2s$ and $2s+1$, and, for every $i<j \leq s$ such that $2i$ and $2j+1$ are not unused, add $(2i,2j+1)$ to  $f(p)[s]$.
Notice that:
\begin{itemize}
    \item if $(\forall^\infty n)(p(n)=0)$ then $f(p) \cong \mathsf{R}[fin] \oplus \overline{K_\N}$ where $\mathsf{R}[fin]$ is the restriction of $\mathsf{R}$ to a finite set of vertices. Hence $f(p)$ is automorphically trivial and, by \Cref{thm:at-implies-online}, online learnable;
    \item  if  $(\exists^\infty n)(p(n)=1)$, $\mathsf{R}$ is an induced subgraph of $f(p)$. It follows that $f(p)$ is not weakly online learnable: this can be obtained combining \Cref{prop:onlinepacdifferent} and \Cref{lem:almostrandomclosed}.
\end{itemize}

The function $g$ is defined as follows. Given in input $q$, at stage $0$ do nothing. At stage $s+1$, if $q(s+1)=1$ add a copy of $K_s$ to  $f(q)[s]$. If $q(s+1)=0$, then add the first fresh vertex (i.e., a vertex for which we have not decided its membership to $f(q)$ so far) as an isolated vertex to  $f(q)[s]$. Notice that:
\begin{itemize}
\item If $ (\exists^\infty n)(q(n)=1)$ then $g(q) \in \mathfrak{K}_{equiv}$. By \Cref{prop:kequiv},  $g(q)$ is weakly online learnable but not online learnable.
\item  If $(\forall^\infty n)(q(n)=0)$, $g(q) \cong \overline{K_\N} \oplus K_i,\dots,K_j$ for $1<i<j<k$ for some $i,j,k \in \N$. Hence $f(q)$ is automorphically trivial and, by \Cref{thm:at-implies-online}, online learnable.
\end{itemize}

We now prove that the disconnected union of the graphs produced by $f$ and $g$ is the desired reduction.
 Given in input $(p,q) \in 2^{\mathbb{N}} \times 2^{\mathbb{N}}$ notice that,
\begin{itemize}
    \item if $(\forall^\infty n)(p(n)=0)$ and $(\exists^\infty n)(q(n)=1)$, then $f(p) \cong \mathsf{R}[fin] \oplus \overline{K_\N}$ and $g(q)\in \mathfrak{K}_{equiv}$.  A similar strategy to the one used in \Cref{prop:kequiv} to prove that every graph in $\mathfrak{K}_{equiv}$ is weakly online learnable also shows that $f(p)\oplus g(q)$ is weakly online learnable. On the other hand such a graph is clearly not automorphically trivial and hence not online learnable (\Cref{thm:at-implies-online}).
    \item if $(\exists^\infty n)(p(n)=1)$ then $\mathsf{R}$ is an induced subgraph of $f(p) \oplus g(q)$ and so it is not weakly online learnable;
    \item if $(\forall^\infty n)(q(n)=0)$ and $(\forall^\infty n)(q(n)=0)$ then $f(p) \cong \mathsf{R}[fin] \oplus \overline{K_\N}$ 
    and $g(q) \cong \overline{K_\N} \oplus K_i,\dots,K_j$ for $1<i<j<k$ for some $i,j,k \in \N$. The graph $f(p) \oplus g(q)$ is automorphically trivial and hence, by \Cref{thm:at-implies-online}, online learnable.
\end{itemize}
This concludes the proof.
\end{proof}

\begin{thm}
\label{thm:secondcomplexityD2}
    The class of graphs
    \[\{G : G \text{ is a weakly PAC learnable but not weakly online learnable graph}\}\] is $D_2(\mathbf{\Sigma}_2^0)$-complete
\end{thm}
\begin{proof}
    First notice that the definition of such a class is $D_2(\mathbf{\Sigma}_2^0)$ and, as in the proof of \Cref{thm:firstcomplexityD2}, we reduce the set $\mathsf{EC} \times \mathsf{NEC}$ to such class.

Our desired reduction combines two functions, namely the function $h$ defined in the proof of \Cref{thm:firstcomplexity} and the function $f$ defined in the proof of \Cref{thm:firstcomplexityD2}.  Namely, given in input $(p,q)$ it outputs the disconnected union of $h(p)$ and $f(q)$. To show that this is the desired reduction notice that:
\begin{itemize}
    \item if $(\forall^\infty n)(p(n)=0)$ and $(\exists^\infty n)(q(n)=1)$, then $h(p) \cong \mathcal{R}[fin] \oplus \overline{K_\N}$ and $f(q) \cong \mathsf{R} \oplus \overline{K_\N}$.  Combining \Cref{prop:onlinepacdifferent} and \Cref{lem:almostrandomclosed} we obtain that $h(p) \oplus g(q)$ is not weakly online learnable.  To prove that $h(p) \oplus f(q)$ is weakly PAC learnable it suffices to show that it is is not almost-random (and hence, by \Cref{thm:ar-equals-anl}, absolutely non-learnable). To do so notice that neither $h(p)$ nor $f(q)$ are almost random and hence, by \Cref{lem:disjointunionalmostrandom}, $h(p) \oplus f(q)$ is not almost random as well.
    \item if $(\exists^\infty n)(p(n)=1)$ then $h(p)$ contains a copy of the Random graph $\mathcal{R}$. Combining \Cref{cor:randomgraph} and \Cref{lem:almostrandomclosed} we obtain that $h(p) \oplus g(q)$ is not weakly PAC learnable.
    \item if $(\forall^\infty n)(q(n)=0)$ and $(\forall^\infty n)(q(n)=0)$ then $h(p) \cong \mathcal{R}[fin] \oplus \overline{K_\N}$ and $f(q) \cong \mathsf{R}[fin]\oplus \overline{K_\N}$. Hence $h(p) \oplus f(q)$ is automorphically trival and, by \Cref{thm:at-implies-online}, online learnable.
\end{itemize}
This concludes the proof.
\end{proof}
\subsection{Complexity in the computable case}
\label{sec:complexitycomputable}
In this section, we show that the results from the previous subsection can be extended---suitably adapted---to  classes of \emph{computable graphs}, that is, graphs whose edge relation is decidable by a Turing machine. More formally,  fix a computable enumeration of all partial computable functions $(\varphi_e)_{e \in \mathbb{N}}$ as well as a computable bijection  $(n,m)\mapsto\langle n, m\rangle$ from $\mathbb{N}\times \mathbb{N}$ to  $\mathbb{N}$.

\begin{dfn}
A number $e\in \mathbb{N}$ is a \emph{computable index} for a graph $G$ with $V(G)=\mathbb{N}$ if, for all $n$ and $m\in\mathbb{N}$,
\[
\varphi_e(\langle n,m\rangle)=\begin{cases} 1 &\text{$(n,m)\in E(G)$}\\
    0 &\text{otherwise}.
\end{cases}
\]
A graph $G$ is \emph{computable}, if $G$ has a computable index.
\end{dfn}

%we say that $\varphi_e$ is a computable graph if, fixed a computable bijection from pairs of natural numbers to natural numbers there is computable procedure to decide whether $\varphi_e((i,j))=1$ or $\varphi_e((i,j))=1$

We study the complexity of the following index sets:
\begin{equation}
\tag{$\dagger$}
    \label{eq1}
    \{e : \varphi_e \text{ is an } \mathfrak{P}\text{-learnable graph} \}  \text{ and }
    \{e : \varphi_e \text{ is an } \mathfrak{P}\text{-learnable but not }\mathfrak{P}'\text{-learnable graph} \},
\end{equation}
where $\mathfrak{P}$ and $\mathfrak{P}'$ are one of the learning paradigms introduced in this paper.

\smallskip

The sets in \Cref{eq1} are subsets of the natural numbers are subsets of the natural numbers, and it is customary to study their complexity using the effective analogue of Wadge reducibility, namely \emph{$m$-reducibility}. A set $A\subseteq \N$ is said to be $m$-reducible to a set $B\subseteq \N$ if there exists a computable function $f$ such that $x \in A$ if and only if $f(x) \in B$. The notions of $\Gamma$-hardness and $\Gamma$-completeness  are defined similarly to those in Section \ref{sec:complexityprel}.
%defined analogously to \Cref{sec:complexity}, with  $\Gamma$ now   denoting a collection of subsets of natural numbers. 

To measure the complexity of such sets, we use the \emph{arithmetical hierarchy}, the effective counterpart of (the finite levels of) the Borel hierarchy defined in \Cref{sec:complexityprel}, which classifies subsets of $\mathbb{N}$ that are definable in the language of first-order arithmetic according to the logical form of their defining formulas. Specifically, a set $A \subseteq \mathbb{N}$ is in the class $\Sigma^0_n$ if there exists a computable relation $R \subseteq \mathbb{N}^{n+1}$ such that:
\[
x \in A \iff (\exists y_1 \forall y_2 \ldots Q y_n)\, R(x, y_1, \ldots, y_n),
\]
where $Q$ is an existential quantifier if $n$ is odd, and a universal quantifier if $n$ is even. A set $A$ is in $\Pi^0_n$ if its complement $\omega \smallsetminus A$ is in $\Sigma^0_n$. Finally, a set $X\subseteq \mathbb{N}$ is in $D_2(\Sigma_2^0)$, if $X$ coincides with $A\smallsetminus B$ for some $A,B\in \Sigma^0_2$. The class $\Sigma^0_1$ coincides with the computably enumerable (c.e.) sets, and $\Pi^0_1$ corresponds to the co-c.e.\ sets.  The $\Sigma^0_2$ and $\Pi^0_2$ sets also admit a natural computable approximation. Namely, 
\begin{itemize}
    \item A set $X\subseteq \mathbb{N}$ is $\Sigma^0_2$ if and only if there exists a computable function $g:\mathbb{N}\to \{0,1\}$ so that $x\in X$ if and only if $\liminf_{s\rightarrow \infty} g(\langle x,s\rangle)=1$; that is, $x\in X$ if and only if $g(\langle x,t\rangle)=1$ for all sufficiently large $t$.
    \item  Similarly, a set $X\subseteq \mathbb{N}$ is $\Pi^0_2$ if and only if there exists a computable function $g:\mathbb{N}\to \{0,1\}$ so that $x\in X$ if and only if $\limsup_{s\rightarrow \infty} g(\langle x,s\rangle)=1$; that is, $x\in X$ if and only if $g(\langle x,t\rangle)=1$ for infinitely many $t$. 
\end{itemize}

To transfer our complexity results to the computable setting, we will use some well-known benchmark sets (see \cite[Theorem IV.1.5]{Soare1987}). 

\begin{lem} The following holds:
\begin{itemize}
\item The set $\mathsf{Fin}:=\{e: W_e \text{ is finite}\}$  is $\Sigma^0_2$-complete and its complement $\mathsf{Inf}:=\{e: W_e \text{ is infinite}\}$ is $\Pi^0_2$-complete;
\item The set $\mathsf{Fin} \times \mathsf{Inf}:=\{\langle e,i\rangle : W_e \in \mathsf{Fin} \text{ and }W_i \in \mathsf{Inf}\}$ is  $D_2(\Sigma_2^0)$-complete.
\end{itemize}
\end{lem}

When analyzing the complexity of sets of graphs learnable within a given paradigm, it is natural to restrict attention to indices that actually define graphs, rather than arbitrary computable functions. However, this natural requirement does not come \lq\lq for free\rq\rq.
Indeed, $\mathsf{CG}:=\{e : e \text{ is the index of a computable graph} \}$ is a $\Pi_2^0$-complete set\footnote{This can be easily seen noticing that $e$ is the index for a computable graph if and only if $e$ is the index of a total computable function, and the set $\mathsf{Tot}:=\{e : \varphi_e \text{ is a total computable function}\}$ is a well-known $\Pi_2^0$ set.}. This implies that the sets in \Cref{eq1} are at least $\Pi_2^0$-hard. This is undesired as  most of the learning criteria we considered so far have a $\Sigma_2^0$ definition. 
To address this issue, we will follow a common practice in computable structure theory (see e.g.\ \cite{calvert2006index}): analyzing the \lq\lq interior complexity\rq\rq\ of sets, as formally defined in \Cref{def:complexityinterior}.
Notice that this is not needed for the sets
$\{e : \varphi_e \text{ is an absolutely non-learnable graphs}\}$ and the sets of the second type in \Cref{eq1} as they are already $\Pi_2^0$-hard.

\begin{dfn}
\label{def:complexityinterior}
Let $A\subseteq B\subseteq \mathbb{N}$ and let $\Gamma$  be a class of sets:
\begin{itemize}
    \item We say that $A \in \Gamma$ \emph{within $B$} if and only if $A = R \cap B$ for some $R \in \Gamma$.
    \item We say that $ S \leq_m A$ \emph{within $B$} if and only if there is a computable $f : \mathbb{N} \rightarrow B$ such
that for all $n$ we have $n \in S \iff f (n) \in A$.
\item We say that $A$ is $\Gamma$-complete \emph{within $B$} if and only if $A$ is $\Gamma$ \emph{within $B$} and for every $S \in \Gamma$ we have $S \leq_m A$ \emph{within $B$}.
\end{itemize}

\end{dfn}
Notice that removing the \lq\lq within $B$\rq\rq\ part, one obtains the usual notions of being in $\Gamma$, $m$-reducibility and $\Gamma$-completeness. We also mention that in this paper, $A$ is a set like the ones in \Cref{eq1},
    $B$ is $\mathsf{CG}$ and
     $\Gamma \in \{\Sigma_2^0,\Pi_2^0, D_2(\Sigma_2^0)\}$.

\begin{rem}
    Studying the complexity of a set within another (as done in \Cref{def:complexityinterior}) is an idea that is present also in the literature about PAC/online of finite hypothesis classes. Researchers in this area often refers to the study of \emph{promise problems}, namely learning problems were the learner is promised to receive meaningful inputs and not arbitrary ones (see, \cite{Goldreich2006} for a survey on promise problems).  It is not hard to check that \Cref{def:complexityinterior} captures exactly the complexity of promise problems in the infinite setting: more precisely, in our setting, we \lq\lq promise\rq\rq\ that the input will be an index of a computable graph.
\end{rem}

%Furthermore, in this setting,  the sets we analyze are be subsets of the \emph{arithmetical hierarchy}, the effective analog of the finite levels of Borel hierarchy. This hierarchy is defined on subsets of natural numbers and, by analogy with the Borel classes $\boldsymbol{\Sigma}_n^0$, $\boldsymbol{\Pi}_n^0$ and $\boldsymbol{\Delta}_n^0$ we use the \lq\lq lightface\rq\rq\ symbols ${\Sigma}_n^0$, ${\Pi}_n^0$ and ${\Delta}_n^0$.

With minor modifications to the proofs, one can obtain the effective analogs of \Cref{thm:firstcomplexity}, \Cref{thm:firstcomplexityD2} and \Cref{thm:secondcomplexityD2}.

\begin{thm}
\label{theorem:firsteffectivitazion}
The following holds:
\begin{itemize}
    \item The sets
\[\{e : \varphi_e \text{ is a computable }\mathfrak{P}\text{-learnable graph}\}\]
where $\mathfrak{P} \in \{\text{online, PAC, weakly online, weakly PAC}\}$ are ${\Sigma}_2^0$-complete within $\mathsf{CG}$.
\item The set $\{e : \varphi_e \text{ is a computable absolutely non-learnable graph}\}$
 is $\Pi_2^0$-complete.
 \item   The sets 
     \[\{e : \varphi_e \text{ is a computable  weakly online learnable but not online learnable graph}\} \text{ and }\]
       \[\{e : \varphi_e \text{ is a computable weakly PAC learnable but not weakly online learnable graph}\}\]
      are $D_2({\Sigma}_2^0)$-complete.
\end{itemize}

\end{thm}
\begin{proof}[Sketch of the proof]
    Notice that the reductions $h,f,g$ defined in the proofs of \Cref{thm:firstcomplexity}, \Cref{thm:firstcomplexityD2} and \Cref{thm:secondcomplexityD2} are computable. This is immedate as all of these reductions are defined via an algorithm producing a graph based on a finite fragment of some $p \in 2^\N$.

    Secondly, notice that identifying any computably enumerable set $W_e$ with index $e$ via its characteristic function $\mathsf{Fin}$ is the same as $\mathsf{EC}$, $\mathsf{Inf}$ is the same as $\mathsf{NEC}$, restricted to the computable sequences in $2^\N$ and $\mathsf{Fin} \times \mathsf{Inf}$ is the same as $\mathsf{EC} \times \mathsf{NEC}$. We can conclude that:
    \begin{itemize}
        \item   the reduction $h$ (defined in \Cref{thm:firstcomplexity}) is a computable reduction both from $\mathsf{Fin}$ to the computable hypothesis classes of graphs that are online learnable and from $\mathsf{Inf}$ to the computable hypothesis classes of graphs that are absolutely non-learnable;
        \item  the reduction $(e,i)\mapsto f(e) \oplus g(i)$ (defined in \Cref{thm:firstcomplexityD2}) is a computable reduction from $\mathsf{Fin} \times \mathsf{Inf}$ to the computable hypothesis classes of graphs that are weakly online learnable but not online learnable;
         \item  the reduction $(e,i)\mapsto h(e) \oplus g(i)$ (defined in \Cref{thm:secondcomplexityD2}) is a computable reduction from $\mathsf{Fin} \times \mathsf{Inf}$ to the computable hypothesis classes of graphs that are weakly PAC learnable but not weakly online learnable;
    \end{itemize}
    This concludes the sketch of the proof.
\end{proof}

Notice that the complexity of the set 
$\{e : e \text{ is an index for a computable PAC learnable graph}\}$
     alignes with the result obtained in \cite[Theorem 4.1]{decidingVC} stating that \begin{equation}
     \label{eq2}
     \{e : e \text{ is an index for a computable PAC learnable hypothesis class}\}
 \end{equation}
 is $\Sigma_2^0$-complete within $\mathsf{Tot}$\footnote{Both in \cite[Theorem 4.1]{decidingVC} and \cite{Sterkenburg2022-STEOCO-5} (which presents the same result) the atuhors state the theorem just in terms of $\Sigma_2^0$-completeness. However, their analysis is carried out within a subclass of computable hypothesis classes and one can verify that the complexity results they obtain are referred to promise problems. Using the terminology of \Cref{def:complexityinterior} this means that such set is $\Sigma_2^0$-complete within $\mathsf{Tot}$.}. Here an \emph{hypothesis class}\footnote{In \cite{decidingVC} the author uses the term \lq\lq concept\rq\rq\ instead of \lq\lq hypothesis\rq\rq. In general, a concept class is the class to be learned while the hypothesis class is the set of possible outputs of a learning algorithm. The two classes may not coincide in general, but this is not the case for our case and \cite{decidingVC}. When the two notions coincide the learning is called \emph{proper}.} is a subset $\Hyp$ of $2^X$ for some universe $X$. An hypothesis class $\Hyp=(h_i)_{i \in \N}$ is \emph{computable} if there is a computable function $\varphi$ such that, for $x \in X$ and $i \in \N$ $\varphi(x,i)=1$ if and only if $h_i(x)=1$.
 %as mentioned in \cite[Page 5]{Sterkenburg2022-STEOCO-5}, the analysis is carried out within a subclass of computable hypothesis classes—specifically, one can verify that this restriction corresponds to obtain that such set is $\Sigma_2^0$-complete within $\mathsf{Tot}$.}.

We point out that our approach and the one just described are, in general, incomparable. For instance, there exist computable hypothesis classes that are not contained in any class of the form $\Iso{G}$. A simple example is the class of all elements of $2^{\mathbb{N}}$ that are eventually constant $0$. This cannot be a subset of any $\Iso{G}$, since having the same number of $1$’s in the characteristic function is a necessary condition for two graphs to be isomorphic. Conversely, even when $G$ is a computable graph, the class $\Iso{G}$ may be uncountable, whereas every computable hypothesis class is countable. This highlights a distinctive feature of our framework: the learning problems we consider allow us to meaningfully study uncountable hypothesis classes by representing them through the index of a computable graph. We conclude this section mentioning that an approach considering the computational complexity of uncountable hypothesis classes has been considered in
\cite{contfeatures} using the framework of Weihrauch reducibility. Here the hypothesis space is an arbitrary computable metric
space.

% The reason why \Cref{theorem:firsteffectivitazion} strengthens \cite[Theorem 4.1]{decidingVC} is that is a computable hypothesis class in the sense of \Cref{def:hypclass}.This is almost immediate as every computable PAC learnable graph is automorphically trivial (\Cref{thm:finite-pac-implies-at}) and every $\Iso{G}$ for $G$ being automorphically trivial is by definition countable. Furthermore, as we already discussed, not every computable hypothesis class is contained in $\Iso{G}$ for some $G$. Hence we obtain that the set in \Cref{eq3} is a strict subset of the general case in \Cref{eq2} that remains of the same complexity (i.e., $\Sigma_2^0$-complete).

%Notice that the computable setting has been explored in the context of PAC learning by several authors: to the best of our knowledge, this has never been done for online learning.

\section{Conclusions and future work}
The results of our investigation reveal a scenario in which graphs can be partitioned into four classes (as depicted in \Cref{fig:landscape}), according to which families of their copies can be learned, and in which framework we can learn them. In fact, given a graph $G$, exactly one of the following cases must occur: 
\begin{enumerate}
    \item $G$ is online learnable (or, equivalently, $\bigcup_k \Fiso{k}{G}$ is PAC learnable). 
    \item $G$ is weakly online learnable but not online learnable: in other words, each of the families $\Fiso{k}{G}$ is online learnable, but their union (and, therefore, $\Iso{G}$) is not. 
    \item $G$ is weakly PAC learnable but not weakly online learnable, i.e.~every family $\Fiso{k}{G}$ is PAC learnable, but not online learnable.
    \item $G$ is absolutely PAC learnable, meaning that none of the families $\Fiso{k}{G}$ is PAC learnable.
\end{enumerate}
Note that a complete picture of the precise complexity of such graph classes learnable within a given paradigm is provided.

On the one hand, these results settle in a precise way the boundaries of statistical learning when applied to the problem of learning families of possible labelings of a given graph. On the other hand, the existence of only four distinct classes may suggest that our \lq\lq qualitative\rq\rq\ approach is too coarse to really serve as a measure of the complexity of the family of presentations of a graph. We therefore propose to complement it with a more \lq\lq quantitative\rq\rq\ approach, in order to refine our classification: in this regard, the complexity of online learnable graphs might be compared with respect to the Littlestone dimension of their family of presentations of a graph, while that of weakly online (respectively, PAC) learnable graphs according to how fast the function $k \mapsto \Ld{\Fiso{k}{G}}$ (respectively, $k \mapsto \VC{\Fiso{k}{G}}$) grows. To this end, it is critically important to provide accurate lower and upper bounds to both the Littlestone dimension and the VC dimension of the families $\Fiso{k}{G}$, at least when $G$ belongs to well-studied classes of graphs.

\bibliographystyle{plain}
\bibliography{online}

\end{document}